\documentclass[sigconf]{acmart}

\usepackage{mathtools}
\usepackage{graphicx}
\usepackage{booktabs}
\usepackage[ruled,vlined]{algorithm2e}
\usepackage{cleveref}
\usepackage{xcolor}
\usepackage{subcaption}
\usepackage{multirow}
\usepackage{bm}
\usepackage{colortbl}
\usepackage{tikz}

\newcommand{\Loss}{\mathcal{L}}

\newcommand{\Z}{\mathcal{Z}}
\newcommand{\R}{\mathbb{R}}
\newcommand{\N}{\mathcal{N}}
\newcommand{\E}{\mathbb{E}}

\DeclareMathOperator*{\argmin}{arg\,min}
\definecolor{Gray}{rgb}{0.88, 0.88, 0.88}

\AtBeginDocument{%
  }

\copyrightyear{2026}
\acmYear{2026}
\setcopyright{cc}
\setcctype{by}
\acmConference[MM '26]{Proceedings of the 34th ACM International Conference on Multimedia}{November 10--14, 2026}{Rio de Janeiro, Brazil}
\acmBooktitle{Proceedings of the 34th ACM International Conference on Multimedia (MM '26), November 10--14, 2026, Rio de Janeiro, Brazil}
\acmDOI{10.1145/3767308.3835243}
\acmISBN{979-8-4007-2213-4/2026/11}

\begin{document}

\title{Hybrid-Domain Posterior Sampling for Inverse Problems via Latent Flow Matching}


\author{Hongjie Wu}  \orcid{0009-0007-3203-1521}
 \email{wuhongjie0818@gmail.com}
\authornotemark[1] 
\affiliation{
 \institution{College of Computer Science,\\ Sichuan University} 
 \city{Chengdu} \country{China}
 }

\author{Yiping Xie}  \orcid{0009-0004-8913-129X}
 \email{2020141460037@stu.scu.edu.cn}  
 \authornote{Equal contribution.}
\affiliation{
 \institution{College of Computer Science,\\ Sichuan University} 
 \city{Chengdu} \country{China}
 } 

\author{Jiancheng Lv} 
\authornote{Corresponding author.}
\orcid{0000-0001-6551-3884}
 \email{lvjiancheng@scu.edu.cn}
\affiliation{
\institution{College of Computer Science,\\ Sichuan University}  
 \city{Chengdu}
 \country{China}
 }
 

\begin{abstract}

Latent Flow Models have revolutionized compressed-space image synthesis, yet their application to high-fidelity inverse problems remains bottlenecked. 
In this paper, we trace this dilemma to a fundamental geometric limitation of pre-trained autoencoders, which we term \emph{First-Order Manifold Blindness}. Severe decoder compression (e.g., retaining only $\sim\!2\%$ of the original degrees of freedom) produces a rank-deficient Jacobian, rendering high-frequency measurement residuals in its orthogonal complement invisible to latent gradients even when the decoder can represent the target image.
To overcome this bottleneck, we propose Hybrid-Domain Posterior Sampling (HDPS), a decoupled inference framework that disentangles physical measurement consistency from semantic prior modeling. 
HDPS diverges into the pixel space, leveraging Langevin dynamics to absorb precise orthogonal measurement gradients, and subsequently projects these structural corrections back onto the generative manifold.
An optimization-based latent alignment is introduced to filter pixel-space artifacts while avoiding the semantic drift of direct encoding.
Extensive experiments on diverse inverse problems demonstrate that HDPS establishes a new state-of-the-art, successfully recovering the high-frequency structural precision that latent-only solvers inherently discard.
The code is available at \href{https://github.com/74587887/HDPS}{https://github.com/74587887/HDPS}.

\end{abstract}





\maketitle

\section{Introduction}
\label{sec:introduction}

Recent advances in generative modeling have been driven by the shift from pixel space~\cite{ho2020denoising,song2021scorebased,nichol2021improved} to latent-space~\cite{vahdat2021score,rombach2022high,zhang2023adding}. Notable architectures such as Stable Diffusion~3~\cite{esser2024scaling}, 3.5~\cite{sd35} and FLUX~\cite{labs2025flux} employ \emph{Flow Matching} (FM)~\cite{lipman2023flow,liu2023flow,martin2025pnpflow} on compressed latent representations, achieving state-of-the-art synthesis with significantly reduced computational cost~\cite{rombach2022high,he2023iterative} and straighter generation trajectories~\cite{pourya2025flower}.
Following this success, there has been a surge of interest in leveraging these pre-trained Latent Flow Models (LFMs) for inverse problems~\cite{tarantola2005inverse,chen2021equivariant}, which aim to recover a clean image $\bm{x}$ from noisy or corrupted measurements $\bm{y}$, demonstrating superior performance across diverse benchmarks and setting new state-of-the-art results~\cite{kim2025flowdps,erbach2025FLAIR,park2025flowlps}, particularly in high-resolution image restoration.


While the latent flow model naturally enforces a generative prior, the prevailing approach to incorporating data consistency is to optimize a latent code $\bm{z}$ to minimize a measurement loss $\Loss(\bm{z}) =\|\mathcal{A}(D(\bm{z})) - \bm{y}\|^2$, where $D$ is the pre-trained decoder.
This strategy implicitly assumes that physical measurement constraints can be effectively back-propagated through $D$.
However, because the latent dimension $d$ is drastically smaller than the image dimension $n$ (often retaining only $\sim\!2\%$ of the original degrees of freedom), the decoder Jacobian $J_D \in \mathbb{R}^{n \times d}$ is severely rank-deficient.
As a result, any component of the measurement gradient lying in the orthogonal complement of the Jacobian's column space---typically high-frequency residuals critical for restoration---is mathematically \emph{invisible} to latent updates.

We formalize this as \emph{First-Order Manifold Blindness} (Sec.~\ref{sec:motivation}): the optimizer stagnates even when the decoder theoretically has the capacity to represent the target image, because first-order gradient updates cannot access these required repair directions.
Empirical analysis (shown in Fig.~\ref{fig:problems}) confirms that latent gradients suffer from early optimization stagnation, forcing outcomes that are either unnaturally over-smoothed or geometrically misaligned, thus degrading reconstruction quality.


These observations reveal a fundamental domain mismatch: 
\textbf{\emph{
the prior naturally resides in the latent space, but the measurement likelihood is defined in the pixel space}}.
Forcing pixel-level physics through a compressed bottleneck fundamentally compromises solver fidelity.
This motivates our core design principle: \emph{use each space for what it does best}---the pixel space for measurement consistency, and the latent space for prior modeling.

To resolve this, we propose \textbf{Hybrid-Domain Posterior Sampling (HDPS)}, a decoupled inference framework systematically designed to bypass the manifold blindness bottleneck. HDPS rejects the single-domain optimization paradigm, strategically alternating between two explicitly separated roles. First, we transition into the uncompressed pixel space, utilizing Langevin dynamics to absorb the more precise measurement correction. This mathematically liberates the update step, allowing the recovery of high-frequency structural details strictly orthogonal to the decoder manifold. Second, we identify that directly encoding the corrected image (i.e., $\bm{z} = E(\bm{x})$) introduces uncontrollable semantic drift, as the pre-trained encoder maps non-Gaussian artifacts from pixel-space adjustments into corrupted latent features. To address this, we perform the projection via \emph{optimization-based latent alignment}, which maps the corrections back onto the generative prior while filtering out pixel-space artifacts without semantic drift. We further show theoretically that HDPS resolves manifold blindness through a second-order mechanism that implicitly leverages the decoder's non-linear curvature (Theorem~\ref{thm:resolution}).



Our contributions are summarized as follows:
\begin{itemize}
  \item  We formally identify and mathematically prove the intrinsic limitations of composite back-propagation ($\mathcal{A} \circ D$) in latent-space inverse solvers. We demonstrate that latent updates are fundamentally blind to high-frequency residuals orthogonal to the decoder Jacobian.
  \item  We propose a novel algorithm HDPS that circumvents the geometric bottleneck by strictly decoupling physical measurement consistency (executed via pixel-space Langevin dynamics) from semantic prior evolution (governed by latent flow matching).
  \item We demonstrate that latent projection via decoder inversion optimization is vastly superior to direct encoding, acting as a structural filter that discards off-manifold pixel artifacts without suffering semantic drift.
  \item  Extensive evaluation across five diverse and challenging inverse problems demonstrates that HDPS significantly outperforms existing latent-only and baseline decoupled solvers. 
\end{itemize}

\section{Related Work}
\label{sec:related}

\paragraph{Pixel-space algorithms for inverse problems.}
The first wave of generative inverse solvers operated entirely in pixel space~\cite{chung2022improving,chung2024decomposed,wang2023zeroshot,wu2024diffusion}.
Diffusion Posterior Sampling (DPS)~\cite{chung2023diffusion} approximates the posterior score by combining the unconditional score with a likelihood gradient.
DDRM~\cite{kawar2022denoising} exploits the SVD of the forward operator for closed-form conditional updates.
$\Pi$GDM~\cite{song2023pseudoinverseguided} improves the likelihood gradient via pseudoinverse projections, while Plug-and-Play methods~\cite{zhang2025decoupling,zhu2023denoising} integrate physical operators with learned denoisers via ADMM.
DAPS~\cite{zhang2025improving} introduces an annealing strategy that decouples the prior score from the likelihood gradient, performing Langevin sampling at each noise level.
These methods benefit from a well-defined likelihood in pixel space~\cite{wu2024principled,wu2025enhancing}, but pixel-space diffusion models are computationally expensive at high resolution.
Moreover, the community has increasingly shifted toward latent architectures~\cite{chung2024prompttuning,zhang2024flow} with stronger semantic priors, leaving pixel-space models at a representational disadvantage for complex, high-resolution imagery.

\paragraph{Latent-space methods for inverse problems.}
Recent works adapt posterior sampling to latent diffusion~\cite{askari2025latent,rout2024beyond,zilberstein2025repulsive} and flow models~\cite{ben2024d,yan2025fig}.
PSLD~\cite{rout2023solving} optimizes intermediate latent variables for data consistency during reverse diffusion.
ReSample~\cite{song2024resample} enforces hard data consistency via optimization, followed by stochastic resampling onto the noisy manifold.
In the flow matching paradigm~\cite{liu2023flow,lipman2023flow}, FlowChef~\cite{patel2025flowchef} exploits straight trajectories of rectified flows for gradient-free steering via gradient skipping.
FlowDPS~\cite{kim2025flowdps} derives a flow-version of Tweedie's formula to integrate likelihood gradients into the flow ODE.
FLAIR~\cite{erbach2025FLAIR} introduces a training-free variational framework with deterministic trajectory adjustments.
These methods inherit the efficiency and semantic strength of models like Stable Diffusion~3~\cite{esser2024scaling} and FLUX~\cite{labs2025flux}, but share a critical limitation: enforcing measurement consistency requires back-propagating gradients through the composite operator $\mathcal{A} \circ D$, which is ill-conditioned due to decoder nonlinearity.


\section{Preliminaries and Motivation}
\label{sec:preliminaries}

\subsection{Problem Setup and Latent Flow Matching}
\label{sec:setup}

We consider the recovery of an unknown signal $\bm{x} \in \R^n$ from noisy measurements governed by the forward model:
\begin{equation}
  \bm{y} = \mathcal{A}\bm{x} + \bm{\epsilon}, \quad \bm{\epsilon} \sim \N(\bm{0}, \sigma_y^2 \bm{I}),
  \label{eq:measurement}
\end{equation}
where $\mathcal{A}$ is a known forward degradation operator and $\sigma_y$ is the noise level.
From a Bayesian perspective, we characterize the posterior distribution $p(\bm{x} \mid \bm{y}) \propto p(\bm{y} \mid \bm{x})\, p(\bm{x})$, where the likelihood enforces data fidelity via $p(\bm{y} \mid \bm{x}) \propto \exp\!\left(-\frac{1}{2\sigma_y^2}\|\mathcal{A}\bm{x} - \bm{y}\|^2\right)$, and the prior $p(\bm{x})$ encodes natural image statistics.

To efficiently model the complex prior $p(\bm{x})$, we operate in a compressed latent space induced by a pre-trained VAE~\cite{kingma2013auto,rombach2022high} with encoder $E\colon \R^n \to \R^d$ and decoder $D\colon \R^d \to \R^n$ ($d \ll n$). We assume generation in the latent space $\Z = \R^d$ followed by decoding back to the pixel space via $D$.

A Flow Matching (FM) model~\cite{lipman2023flow} learns a time-dependent vector field $v_\theta(\bm{z}, t)$ that transports a simple Gaussian prior $p_1 = \N(\bm{0}, \bm{I})$ to the latent data distribution $p_0 = E_\# p_{\text{data}}$. This is achieved by defining an Ordinary Differential Equation (ODE):
\begin{equation}
  \frac{\mathrm{d}\bm{z}_t}{\mathrm{d}t} = v_\theta(\bm{z}_t, t), \quad \bm{z}_1 \sim p_1.
  \label{eq:ode}
\end{equation}
Integrating Eq.~\eqref{eq:ode} backward from $t=1$ to $t=0$ yields samples from the data distribution.
To train $v_\theta$, a linear conditional interpolation path is defined between $\bm{z}_0 \sim p_0$ and $\bm{z}_1 \sim p_1$:
\begin{equation}
  \bm{z}_t = (1 - t)\bm{z}_0 + t \bm{z}_1,
\end{equation}
which has a target velocity $\dot{\bm{z}}_t = \bm{z}_1 - \bm{z}_0$. The network is trained to regress this target by minimizing:
\begin{equation}
  \Loss_{\text{FM}}(\theta) = \E_{t, \bm{z}_0, \bm{z}_1}\! \left[\|v_\theta(\bm{z}_t, t) - (\bm{z}_1 - \bm{z}_0)\|^2\right].
\end{equation}
After training, integrating the learned ODE and decoding through $D$ provides an efficient generative process.

\begin{figure*}[t]
    \centering
    \begin{subfigure}[b]{0.48\linewidth}
        \centering
        \includegraphics[width=0.48\linewidth]{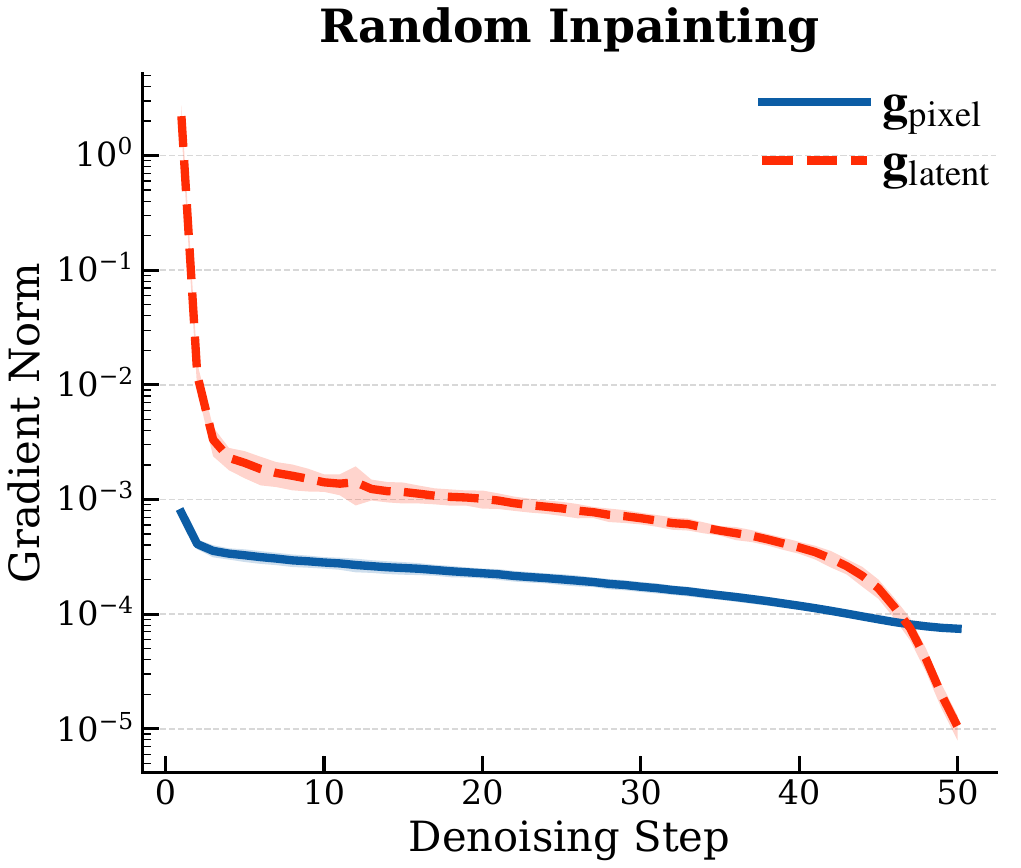}
        \hfill
        \includegraphics[width=0.48\linewidth]{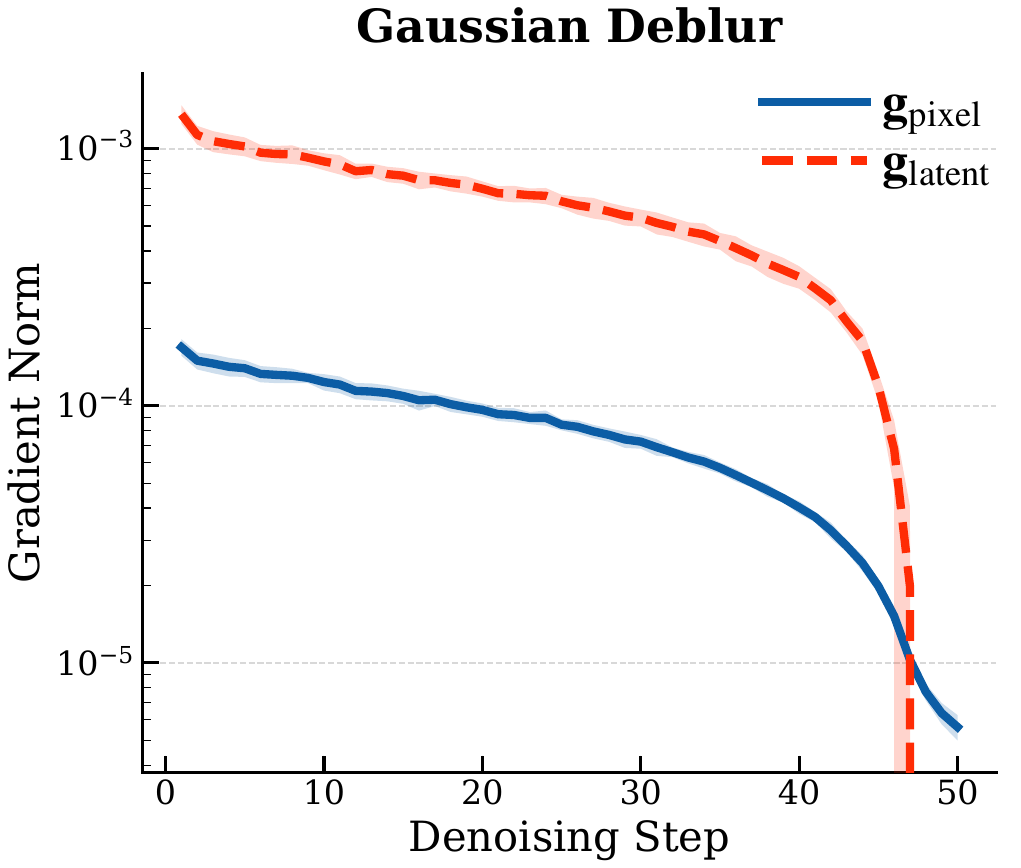}
        
        \vspace{3pt} 
        
        \includegraphics[width=0.48\linewidth]{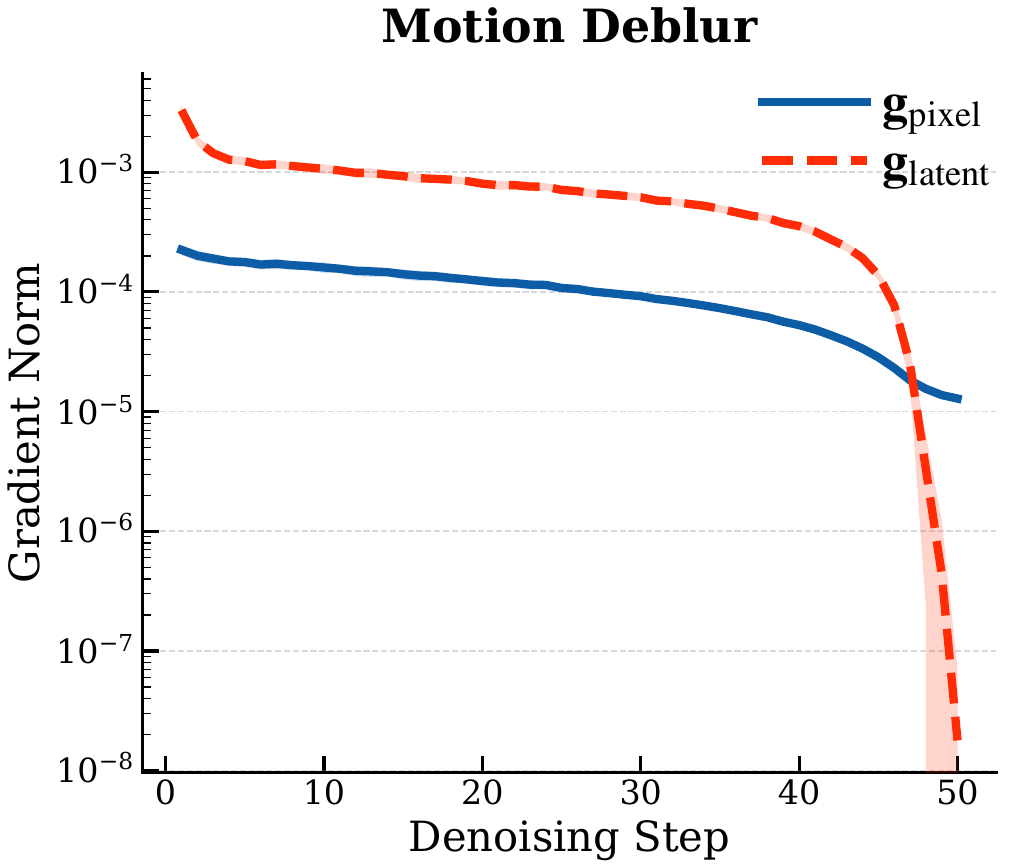}
        \hfill
        \includegraphics[width=0.48\linewidth]{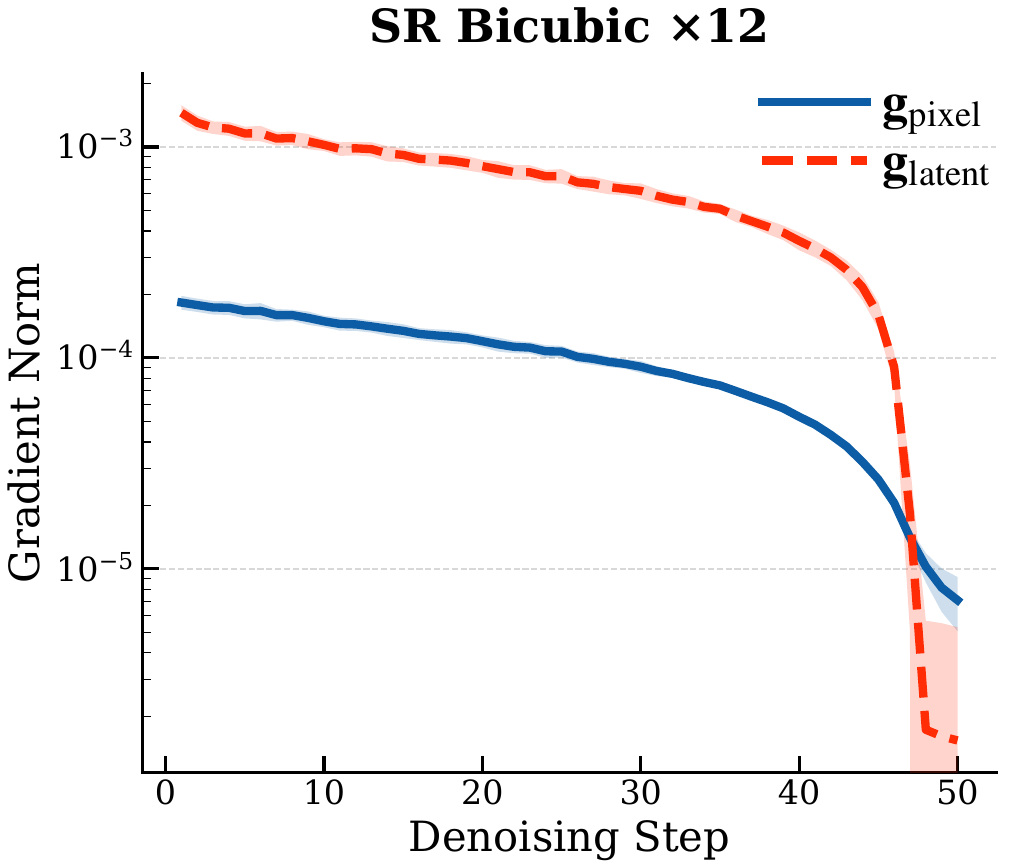}
        
        \vspace{-2pt} 
        \caption{Gradient norm evolution.}
        \label{fig:norms_all}
    \end{subfigure}
    \hspace{1pt} 
    \begin{subfigure}[b]{0.51\linewidth}
        \centering
        \includegraphics[width=\linewidth]{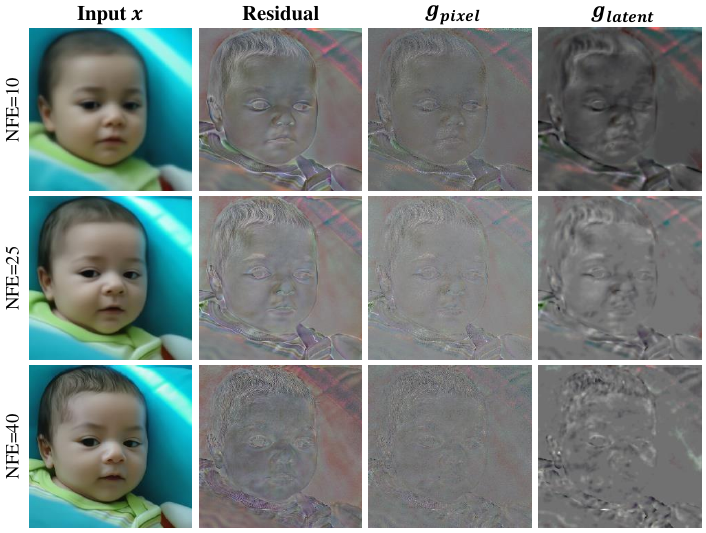}
        \vspace{-2pt}
        \caption{Gradient visualization.}
        \label{fig:vis_large}
    \end{subfigure}

    \caption{\textbf{Latent vs.\ pixel-space gradient dynamics.} (a)~Latent gradients decay by orders of magnitude and vanish in late stages, while pixel gradients remain stable. (b)~Pixel gradients are spatially precise; latent gradients lose structural information after passing through the decoder bottleneck.}
    \vspace{-0.3cm}
    \label{fig:problems}
\end{figure*}

\subsection{Why Latent-Only Optimization Fails}
\label{sec:motivation}

While the generative prior is inherently maintained by the flow trajectory~\cite{kim2025flowdps,park2025flowlps} (or an explicit regularizer $R(\bm{z})$~\cite{erbach2025FLAIR}), the prevailing strategy for enforcing data consistency is \emph{latent optimization}:
\begin{equation}
  \bm{z}^* = \argmin_{\bm{z}} \Loss_{\text{rec}}(\bm{z}), \quad \Loss_{\text{rec}}(\bm{z}) = \tfrac{1}{2}\|\mathcal{A}(D(\bm{z})) - \bm{y}\|^2.
  \label{eq:latent_loss}
\end{equation}
The gradient is $\nabla_{\bm{z}} \Loss_{\text{rec}} = J_D^\top \mathcal{A}^\top (\mathcal{A}(D(\bm{z})) - \bm{y})$, where $J_D \in \R^{n \times d}$ is the decoder Jacobian.
Although this isolates the measurement update properly, geometric analysis reveals a fundamental structural flaw: optimization occurs in a severely rank-deficient representation space. 
For instance, under widely used architectures (e.g., $4$ latent channels and an $8\times$ spatial downsampling factor), an image of resolution $H \times W$ possesses $n = 3HW$ pixel dimensions but only $d = 4(H/8)(W/8) = HW/16$ latent dimensions. Thus, $d = n/48$, meaning the latent space retains merely $\approx 2.1\%$ of the original degrees of freedom.
Consequently, the column space $\mathcal{R}(J_D)$ forms merely a low-dimensional tangent subspace within the pixel space.
Measurement residuals lying in the orthogonal complement of this subspace are mathematically invisible to latent updates. We formalize this as \emph{First-Order Manifold Blindness}:

\begin{proposition}[First-Order Manifold Blindness]
\label{prop:blindness}
Let \hfil \penalty100 \hfilneg $\bm{g}_{\textup{pixel}} = \mathcal{A}^\top(\mathcal{A}(D(\bm{z})) - \bm{y})$ be the back-projected gradient and decompose it as $\bm{g}_{\textup{pixel}} = \bm{g}_{\parallel} + \bm{g}_{\perp}$, where $\bm{g}_{\parallel} \in \mathcal{R}(J_D)$ and $\bm{g}_{\perp} \in \mathcal{R}(J_D)^\perp$.
Then, $\nabla_{\bm{z}} \Loss_{\textup{rec}} = J_D^\top \bm{g}_{\parallel}$; the component $\bm{g}_{\perp}$ is invisible to any latent update.
\end{proposition}

\noindent This follows directly from $J_D^\top \bm{g}_{\perp} = \bm{0}$ for any $\bm{g}_{\perp} \perp \mathcal{R}(J_D)$: the latent gradient is merely the \emph{projection} of the pixel gradient onto the decoder's tangent space.
High-frequency residuals that the decoder cannot locally represent are thus completely discarded. This inevitably causes the optimizer to stagnate---even if the decoder theoretically possesses the capacity to represent the target image, first-order gradient updates cannot access the necessary orthogonal directions for repair.

\textbf{Empirical evidence.}
Figure~\ref{fig:problems} corroborates this analysis.
Tracking gradient norms (Fig.~\ref{fig:problems}(a)) shows that pixel-space gradients $\nabla_{\bm{x}} \Loss$ maintain stable magnitudes throughout the sampling trajectory, while latent gradients $\nabla_{\bm{z}} \Loss$ fluctuate wildly---decaying by approximately five orders of magnitude in inpainting tasks.
Visualizing the gradient maps (Fig.~\ref{fig:problems}(b)) reveals further degradation: $\nabla_{\bm{x}} \Loss$ is spatially precise with high structural correlation to the residual error, whereas $\nabla_{\bm{z}} \Loss$ appears unstructured and coarse, confirming that the decoder bottleneck erases the fine-grained directional information necessary for reliable guidance.

These observations motivate a decoupled design: perform measurement consistency in the pixel space where gradients are full-rank and stable, and use the latent space strictly for prior evolution.

\begin{figure*}[ht]
  \centering
  \includegraphics[width=\linewidth]{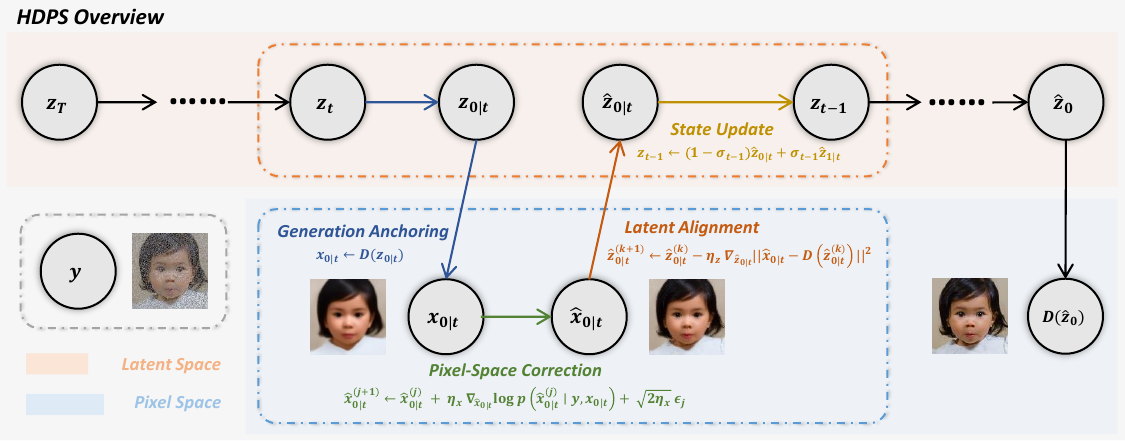}
  \caption{Overview of the proposed decoupled framework.
  Our framework decouples image-space posterior sampling from latent prior modeling, avoiding unstable composite gradients while preserving the expressiveness of the learned prior.}
  \vspace{-0.2cm}
  \label{fig:overview}
\end{figure*}


\begin{algorithm}[t]
\caption{Hybrid-Domain Posterior Sampling}
\label{alg:method}
\SetAlgoLined
\DontPrintSemicolon
\KwIn{Measurement $\bm{y}$, Operator $\mathcal{A}$, Models $v_\theta, (E, D)$, Schedules $\sigma_t, \tau_t$, Params $t_0, \Delta t, N_x, N_z$}
\KwOut{Reconstruction $\hat{\bm{x}}_0$}
\BlankLine
Initialize $\bm{z}_{t_0} \leftarrow \sigma_{t_0}\bm{\epsilon} + (1 - \sigma_{t_0})E(\mathcal{A}^\dagger \bm{y})$\;
\For{$t = t_0, t_0-\Delta t, \dots, 0$}{
  \tcp{1. Generation Anchoring}
  Estimate $\bm{z}_{0|t}$ via flow CFG (Eq.~\eqref{eq:pred_z0t})\;
  Decode anchor $\bm{x}_{0|t} \leftarrow D(\bm{z}_{0|t})$\;
  \tcp{2. Pixel-Space Correction}
  Obtain $\hat{\bm{x}}_{0|t}$ via $N_x$ Langevin steps from $\bm{x}_{0|t}$ (Eq.~\eqref{eq:langevin})\;
  
  \tcp{3. Latent Alignment}
  Obtain $\hat{\bm{z}}_{0|t} $ via $N_z$ steps of decoder inversion from $\bm{z}_{0|t}$ (Eq.~\eqref{eq:alignment})\;
  
  \tcp{4.  State Update}
  Compute next state $\bm{z}_{t - \Delta t}$ with Eq.~\eqref{eq:state_update}\;
}
\Return $\hat{\bm{x}}_0= D(\hat{\bm{z}}_0)$\;
\end{algorithm}

\section{Hybrid-Domain Posterior Sampling}
\label{sec:method}

We now present HDPS, a framework that disentangles measurement correction from prior modeling by alternating between two complementary domains.
Figure~\ref{fig:overview} illustrates the overall pipeline.
At each reverse-process timestep $t$, HDPS executes four stages: (1)~\emph{Generation Anchoring}---the flow model predicts a clean latent estimate; (2)~\emph{Pixel-Space Correction}---Langevin dynamics refine the decoded image to satisfy measurement constraints; (3)~\emph{Latent Alignment}---optimization-based inversion maps the refined image back onto the decoder manifold; and (4)~\emph{State Update}---the corrected latent is integrated into the flow trajectory.
These stages form a strict sequential cycle,
\begin{equation*}
\begin{aligned}
\bm{z}_t &\xrightarrow{\mathrm{flow}} \bm{z}_{0|t} \xrightarrow{D} \bm{x}_{0|t}
\xrightarrow{\mathrm{Langevin}} \hat{\bm{x}}_{0|t} \xrightarrow{\mathrm{alignment}} \hat{\bm{z}}_{0|t}
\xrightarrow{\mathrm{update}} \bm{z}_{t-\Delta t},
\end{aligned}
\end{equation*}
rather than two parallel paths or a post-hoc pixel refinement. In particular, the corrected image is explicitly mapped back by decoder inversion, and the resulting $\hat{\bm{z}}_{0|t}$ enters Eq.~\eqref{eq:state_update}; hence every subsequent latent state is directly conditioned on the preceding pixel-space correction. Although the flow prior is not included as an explicit penalty in the alignment objective, it remains active through the flow-predicted initialization at every outer step and through the state update.
The complete algorithm is summarized in Algorithm~\ref{alg:method}.

\subsection{Generation Anchoring via Flow Prediction}
\label{sec:anchoring}

Given the current noisy state $\bm{z}_t$, we estimate the clean latent code $\bm{z}_{0|t}$ using the flow network $v_\theta$ with classifier-free guidance (CFG)~\cite{ho2021classifierfree}:
\begin{equation}
  \bm{z}_{0|t} = \bm{z}_t - \sigma_t \left[ (1 + \lambda)v_\theta(\bm{z}_t, t, c) - \lambda v_\theta(\bm{z}_t, t, \varnothing) \right],\label{eq:pred_z0t}
\end{equation}
where $c$ is an optional text condition (detailed in Sec.~\ref{sec:experiments}).
The decoded anchor $\bm{x}_{0|t} = D(\bm{z}_{0|t})$ is the model's best estimate of the clean image \emph{before} enforcing data consistency.

\subsection{Pixel-Space Correction}
\label{sec:correction}

To recover high-frequency details lost by the decoder's rank-deficient Jacobian, we perform posterior sampling directly in pixel space.
We treat the decoded anchor $\bm{x}_{0|t}$ as a Gaussian prior and sample from the approximate posterior
\begin{equation}
  p(\hat{\bm{x}}_{0|t} \mid \bm{x}_{0|t}, \bm{y}) \propto \exp\!\left(-\frac{\|\mathcal{A}\hat{\bm{x}}_{0|t} - \bm{y}\|^2}{2\sigma_y^2}\right) \cdot \exp\!\left(-\frac{\|\hat{\bm{x}}_{0|t} - \bm{x}_{0|t}\|^2}{2\tau_t^2}\right),
\end{equation}
where $\tau_t = \sigma_t / \sqrt{1 + \sigma_t^2}$ is an empirical schedule~\cite{ho2022video,song2023loss,zhang2025improving} that permits larger deviations from the anchor at high noise levels and tightens the constraint as $t \to 0$.
We execute $N_x$ steps of Langevin dynamics initialized at $\hat{\bm{x}}^{(0)}_{0|t} = \bm{x}_{0|t}$:
\begin{equation}
  \hat{\bm{x}}_{0|t}^{(j+1)} = \hat{\bm{x}}_{0|t}^{(j)} + \eta_x \!\left[
    \frac{\mathcal{A}^\top(\bm{y} - \mathcal{A}\hat{\bm{x}}_{0|t}^{(j)})}{\sigma_y^2}
    - \frac{\hat{\bm{x}}_{0|t}^{(j)} - \bm{x}_{0|t}}{\tau_t^2}
  \right] + \sqrt{2\eta_x}\,\bm{\xi}_j.
  \label{eq:langevin}
\end{equation}
After $N_x$ iterations, we obtain the corrected estimate $\hat{\bm{x}}_{0|t} = \hat{\bm{x}}_{0|t}^{(N_x)}$.
Unlike latent-space updates, the gradient $\nabla_{\bm{x}} \log p(\bm{y} \mid \bm{x})$ operates directly on pixels, enabling updates orthogonal to $\mathcal{R}(J_D)$ and thus resolving the manifold blindness identified in Proposition~\ref{prop:blindness}.

\paragraph{Relation to DAPS}
Our pixel-space correction draws on the annealed Langevin strategy of DAPS~\cite{zhang2025improving}, but serves a fundamentally different role: rather than being the complete solver, it functions as an \emph{intermediate correction layer} within the latent flow trajectory.
The decoded anchor $\bm{x}_{0|t}$ replaces the diffusion model's denoised estimate, ensuring tight coupling between the two domains.
The novelty lies not in the Langevin step itself, but in its integration with latent-space flow matching and the subsequent latent alignment.

\begin{table*}[t]
    \centering
    \caption{Quantitative results on five linear inverse problems across three datasets at $768 \times 768$ resolution.
    \textbf{Bold} = best; \underline{underline} = second best. All measurements corrupted with $\sigma_n = 0.03$.}
    \vspace{-0.3cm}
    \resizebox{\linewidth}{!}{
    \begin{tabular}{ccccccccccccccccc}
    \toprule
    \rowcolor{Gray}
    & \multicolumn{16}{c}{FFHQ 1k (768 $\times$ 768)}\\
    & \multicolumn{3}{c}{Random Inpainting} & \multicolumn{3}{c}{Gaussian Deblur} & \multicolumn{3}{c}{Motion Deblur} & \multicolumn{3}{c}{SR x12 (Bicubic)} & \multicolumn{3}{c}{SR x12 (Avgpool)}  \\
    \midrule
    Method &
    PSNR$\uparrow$ & SSIM$\uparrow$ &LPIPS$\downarrow$ & 
    PSNR$\uparrow$ & SSIM$\uparrow$ &LPIPS$\downarrow$ & 
    PSNR$\uparrow$ & SSIM$\uparrow$ &LPIPS$\downarrow$ & 
    PSNR$\uparrow$ & SSIM$\uparrow$ &LPIPS$\downarrow$ &
    PSNR$\uparrow$ & SSIM$\uparrow$ &LPIPS$\downarrow$ \\
    \midrule
    LatentDAPS & 27.72 & 0.692 & 0.092 & 25.85 & 0.771 & 0.213 & 25.91 & 0.587 & 0.149 & 26.50 & 0.778 & 0.187 & 26.05 & 0.716 & 0.186      \\
    ReSample & 28.18 & 0.748 & 0.087 & 22.53 & 0.423 & 0.301 & 24.26 & 0.532 & 0.163 & 24.22 & 0.536 & 0.251 & 24.01 & 0.524 & 0.252 &    \\
    FlowChef & 26.95 & 0.757 & 0.181 & 24.99 & 0.706 & 0.230 & 27.12 & 0.756 & 0.158 & 25.59 & 0.715 & 0.220 & 25.51 & 0.717 & 0.216 &    \\
    FlowDPS  & 29.62 & 0.830 & 0.114 & 26.50 & 0.763 & 0.197 & 29.01 & 0.803 & 0.119 & \underline{27.28} & \underline{0.770} & \textbf{0.152} & \textbf{27.11} & \textbf{0.770} & \underline{0.158} &    \\
    FLAIR    & \underline{32.74} & \underline{0.888} & \underline{0.020} & \underline{28.31} & \underline{0.768} & \underline{0.093} & \underline{30.56} & \underline{0.823} & \underline{0.032} & 25.40 & 0.667 & 0.234 & 24.75 & 0.617 & 0.243 &    \\
    HDPS (Ours)     & \textbf{34.62} & \textbf{0.923} & \textbf{0.014} & \textbf{30.06} & \textbf{0.814} & \textbf{0.070} & \textbf{32.83} & \textbf{0.892} & \textbf{0.025} & \textbf{27.74} & \textbf{0.790} & \underline{0.154} & \underline{26.99} & \underline{0.719} & \textbf{0.146} &    \\
    \midrule
    \rowcolor{Gray}
    & \multicolumn{16}{c}{AFHQ 1k (768 $\times$ 768)}\\
    & \multicolumn{3}{c}{Random Inpainting} & \multicolumn{3}{c}{Gaussian Deblur} & \multicolumn{3}{c}{Motion Deblur} & \multicolumn{3}{c}{SR x12 (Bicubic)} & \multicolumn{3}{c}{SR x12 (Avgpool)}  \\
    \midrule
    Method &
    PSNR$\uparrow$ & SSIM$\uparrow$ &LPIPS$\downarrow$ & 
    PSNR$\uparrow$ & SSIM$\uparrow$ &LPIPS$\downarrow$ & 
    PSNR$\uparrow$ & SSIM$\uparrow$ &LPIPS$\downarrow$ & 
    PSNR$\uparrow$ & SSIM$\uparrow$ &LPIPS$\downarrow$ &
    PSNR$\uparrow$ & SSIM$\uparrow$ &LPIPS$\downarrow$ \\
    \midrule
    LatentDAPS & 28.37 & 0.699 & 0.086 & 25.72 & 0.731 & 0.231 & 26.03 & 0.594 & 0.144 & 26.33 & 0.740 & 0.203 & 25.92 & 0.690 & 0.198      \\
    ReSample & 29.08 & 0.752 & 0.085 & 24.12 & 0.509 & 0.250 & 26.05 & 0.609 & 0.120 & 25.22 & 0.576 & 0.232 & 25.01 & 0.567 & 0.234 &    \\
    FlowChef & 26.12 & 0.717 & 0.217 & 24.64 & 0.664 & 0.251 & 26.59 & 0.717 & 0.183 & 25.34 & 0.680 & 0.222 & 25.24 & 0.679 & 0.225 &    \\
    FlowDPS  & 28.74 & 0.793 & 0.149 & 26.38 & 0.733 & 0.234 & 28.55 & 0.776 & 0.148 & \underline{27.17} & \underline{0.743} & \underline{0.177} & \textbf{26.96} & \textbf{0.742} & \underline{0.187} &    \\
    FLAIR    & \underline{32.97} & \underline{0.879} & \underline{0.023} & \underline{28.00} & \underline{0.735} & \underline{0.106} & \underline{30.66} & \underline{0.809} & \underline{0.036} & 25.47 & 0.652 & 0.226 & 25.07 & 0.626 & 0.232 &    \\
    HDPS (Ours)     & \textbf{34.34} & \textbf{0.912} & \textbf{0.016} & \textbf{29.63} & \textbf{0.780} & \textbf{0.078} & \textbf{32.38} & \textbf{0.869} & \textbf{0.030} & \textbf{27.42} & \textbf{0.751} & \textbf{0.170} & \underline{26.90} & \underline{0.697} & \textbf{0.165} &    \\
    \midrule
    \rowcolor{Gray}
    & \multicolumn{16}{c}{DIV2K 0.8k (768 $\times$ 768)}\\
    & \multicolumn{3}{c}{Random Inpainting} & \multicolumn{3}{c}{Gaussian Deblur} & \multicolumn{3}{c}{Motion Deblur} & \multicolumn{3}{c}{SR x12 (Bicubic)} & \multicolumn{3}{c}{SR x12 (Avgpool)}  \\
    \midrule
    Method &
    PSNR$\uparrow$ & SSIM$\uparrow$ &LPIPS$\downarrow$ & 
    PSNR$\uparrow$ & SSIM$\uparrow$ &LPIPS$\downarrow$ & 
    PSNR$\uparrow$ & SSIM$\uparrow$ &LPIPS$\downarrow$ & 
    PSNR$\uparrow$ & SSIM$\uparrow$ &LPIPS$\downarrow$ &
    PSNR$\uparrow$ & SSIM$\uparrow$ &LPIPS$\downarrow$ \\
    \midrule
    LatentDAPS & 24.61 & 0.669 & 0.093 & 20.94 & 0.532 & 0.309 & 21.48 & 0.474 & 0.166 & 21.12 & 0.537 & 0.287 & 20.58 & 0.498 & 0.276      \\
    ReSample & 23.77 & 0.627 & 0.119 & 19.06 & 0.309 & 0.324 & 20.61 & 0.439 & 0.159 & 19.12 & 0.315 & 0.318 & 18.86 & 0.302 & 0.323 &    \\
    FlowChef & 20.85 & 0.525 & 0.314 & 18.89 & 0.436 & 0.363 & 20.98 & 0.522 & 0.262 & 19.59 & 0.465 & 0.315 & 19.35 & 0.460 & 0.321 &    \\
    FlowDPS  & 24.21 & 0.667 & 0.157 & 20.79 & 0.512 & 0.301 & 23.39 & 0.614 & 0.154 & \underline{21.30} & \underline{0.529} & \textbf{0.246} & \textbf{21.08} & \textbf{0.525} & \underline{0.250} &    \\
    FLAIR    & \underline{27.08} & \underline{0.822} & \underline{0.021} & \underline{22.26} & \underline{0.555} & \underline{0.157} & \underline{25.05} & \underline{0.710} & \underline{0.042} & 20.40 & 0.473 & 0.285 & 20.10 & 0.457 & 0.288 &    \\
    HDPS (Ours)     & \textbf{27.90} & \textbf{0.845} & \textbf{0.015} & \textbf{23.46} & \textbf{0.604} & \textbf{0.128} & \textbf{26.33} & \textbf{0.761} & \textbf{0.036} & \textbf{21.60} & \textbf{0.548} & \underline{0.250} & \underline{21.05} & \underline{0.504} & \textbf{0.240} &    \\
    \bottomrule
    \end{tabular}
    }    
    \label{tab:quantitative}
    \vspace{-0.2cm}
\end{table*}

\subsection{Latent Alignment via Decoder Inversion}
\label{sec:projection}

The corrected image $\hat{\bm{x}}_{0|t}$ satisfies measurements but may contain artifacts from the Langevin process.
To restore generative consistency, we project $\hat{\bm{x}}_{0|t}$ back onto the decoder manifold via test-time optimization initialized at $\hat{\bm{z}}^{(0)}_{0|t} = \bm{z}_{0|t}$:
\begin{equation}
  \hat{\bm{z}}^{(k+1)}_{0|t} = \hat{\bm{z}}^{(k)}_{0|t} - \eta_z\,\nabla_{\bm{z}} \|\hat{\bm{x}}_{0|t} - D(\hat{\bm{z}}^{(k)}_{0|t})\|^2,
  \label{eq:alignment}
\end{equation}
and iterated for $N_z$ steps to obtain $\hat{\bm{z}}_{0|t}=\hat{\bm{z}}^{(N_z)}_{0|t}$.

\paragraph{Why not encode directly?}
A natural alternative is direct encoding $\hat{\bm{z}} = E(\hat{\bm{x}})$, which is computationally cheaper but fundamentally suboptimal.
The encoder $E$, trained on clean natural images, is highly sensitive to the non-Gaussian artifacts introduced by pixel-space gradients, often mapping them to incorrect semantic regions of $\Z$.
In contrast, decoder inversion acts as a robust \emph{manifold filter}: it finds the nearest valid latent code that reproduces the measurement-consistent content of $\hat{\bm{x}}$ while aggressively ignoring artifacts the decoder cannot generate.
As shown in our ablations (Sec.~\ref{sec:exp_ablation_decoupling}), this yields a \textbf{+3.4\,dB} improvement over encoder-based projection.

\textbf{Theoretical Justification.}
We now show that the pixel-space correction followed by latent alignment resolves the blindness identified in Proposition~\ref{prop:blindness}.

\begin{theorem}[Resolution of Manifold Blindness]
\label{thm:resolution}
Let $\bm{x}_{0|t} = D(\bm{z}_{0|t})$ and $\bm{g} = \nabla_{\bm{x}} \Loss_{\textup{rec}}(\bm{x}_{0|t}) = \bm{g}_{\parallel} + \bm{g}_{\perp}$ with $\bm{g}_{\parallel} \in \mathcal{R}(J_D)$, $\bm{g}_{\perp} \in \mathcal{R}(J_D)^\perp$.
Assume $D$ is twice differentiable with Hessian tensor $\mathcal{H}_D$.
Consider the HDPS update: $\hat{\bm{x}}_{0|t} = \bm{x}_{0|t} - \eta_x \bm{g}$, $\hat{\bm{z}}_{0|t} = \argmin_{\bm{z}} \|\hat{\bm{x}}_{0|t} - D(\bm{z})\|^2$.
Then the effective update $\Delta \bm{x} = D(\hat{\bm{z}}_{0|t}) - \bm{x}_{0|t}$ satisfies
\begin{equation}
    \langle \Delta \bm{x}, \bm{g}_{\perp} \rangle = \tfrac{1}{2}\eta_x^2 \, \bm{v}^\top \!\left(\bm{g}_{\perp}^\top \mathcal{H}_D\right) \bm{v} + \mathcal{O}(\eta_x^3),
\end{equation}
where $\bm{v} = -(J_D^\top J_D)^{\dagger} J_D^\top \bm{g}_{\parallel}$.
\end{theorem}

\begin{proof}
    Please refer to Appendix~\ref{app:proof}.
\end{proof}
\noindent Unlike pure latent optimization where $\langle \Delta \bm{x}, \bm{g}_{\perp} \rangle = 0$, HDPS achieves a non-zero, second-order update that implicitly leverages the decoder's curvature $\mathcal{H}_D$.
By stepping off-manifold into pixel space and projecting back, HDPS ``bends'' the update direction to naturally incorporate orthogonal corrections---\emph{without} computing the expensive Hessian tensor explicitly.

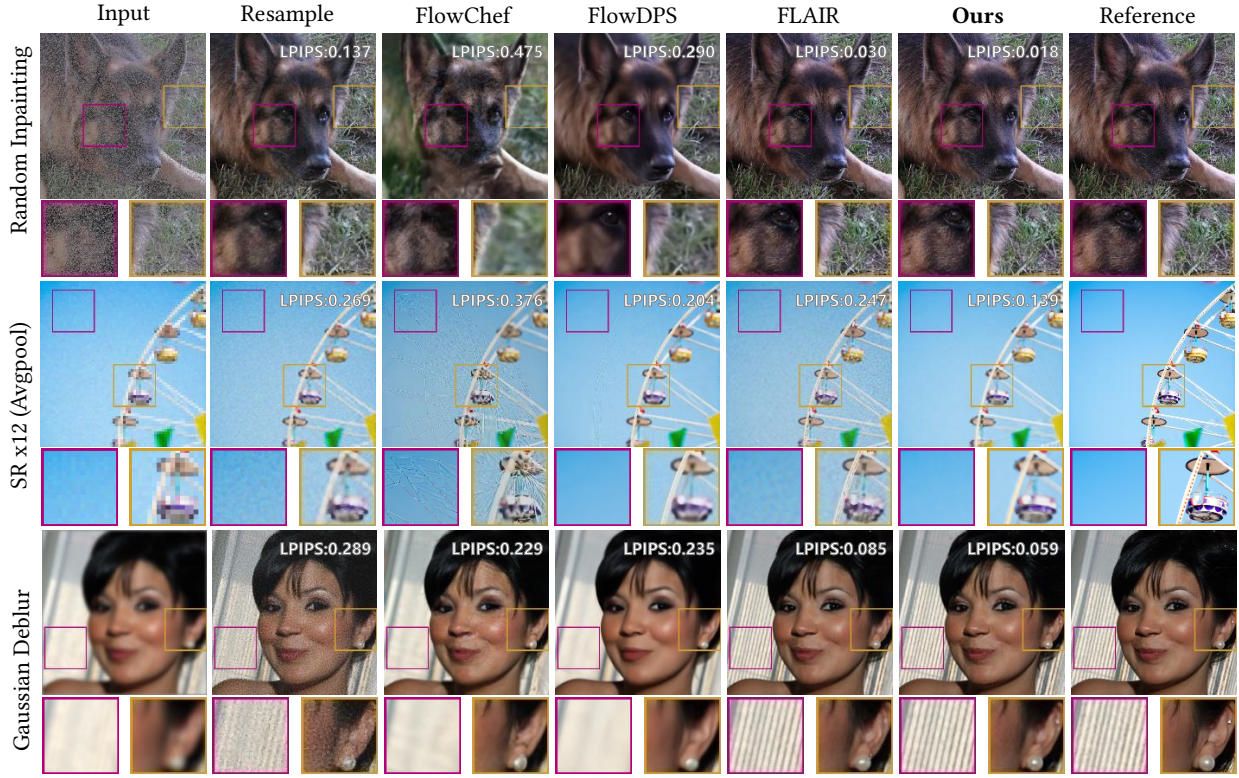
\begin{figure*}[t]
    \centering     
    \begin{tikzpicture}
        \node[anchor=south west,inner sep=0] (image) at (0,0) {\includegraphics[width=0.9\textwidth]{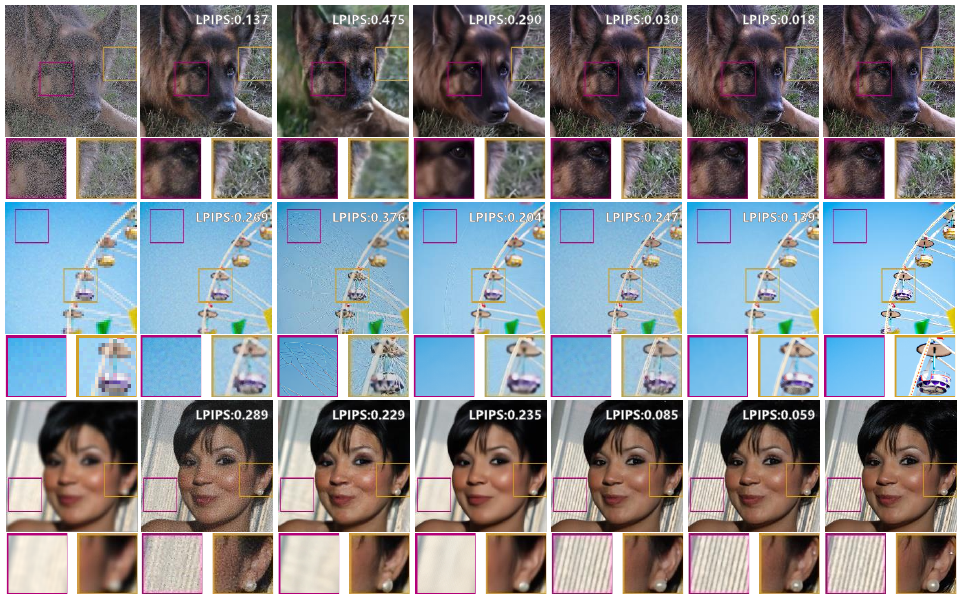}};  
        \begin{scope}[x={(image.south east)},y={(image.north west)}]
            \node at (0.075,1.015) {{Input}};
            \node at (0.21,1.015) {{Resample}};
            \node at (0.355,1.015) {{FlowChef}};
            \node at (0.495,1.015) {{FlowDPS}};
            \node at (0.64,1.015) {{FLAIR}};
            \node at (0.78,1.015) {\textbf{Ours}};
            \node at (0.92,1.015) {{Reference}};
        \end{scope}
        \begin{scope}[x={(image.south east)},y={(image.north west)}]
        \node[rotate=90] at (-0.01,0.85) {{Random Inpainting}};
        \node[rotate=90] at (-0.01,0.50) {{SR x12 (Avgpool)}};
        \node[rotate=90] at (-0.01,0.15) {{Gaussian Deblur}};
    \end{scope}
    \end{tikzpicture}
      \vspace{-0.3cm}
  \caption{Qualitative comparison on five inverse problems. Zoomed-in patches highlight regions where HDPS produces sharper details and fewer artifacts than baselines.}
  \vspace{-0.3cm}
  \label{fig:qualitative}
\end{figure*}

\subsection{State Update}
\label{sec:update}

To robustly preserve the diversity of the posterior distribution and prevent the deterministic ODE from collapsing into local minima due to discretization errors, we integrate the projected semantic content back into the dynamic trajectory using a stochastic injection scheme~\cite{song2021scorebased}. Specifically, we inject structural stochasticity into the noise endpoint $\bm{z}_{1|t}$ by computing $\hat{\bm{z}}_{1|t} = \alpha_t\bm{z}_{1|t} + \sqrt{1-\alpha_t^2} \bm{\epsilon}'$, where $\alpha_t=1-\sigma_t$, $\bm{\epsilon}'\sim \N(\bm{0}, \bm{I})$. The latent state is then securely advanced to $t - \Delta t$:
\begin{equation}
\label{eq:state_update}
  \bm{z}_{t - \Delta t} = (1 - \sigma_{t - \Delta t})\,\hat{\bm{z}}_{0|t} + \sigma_{t - \Delta t}\,\hat{\bm{z}}_{1|t}.
\end{equation}

This mathematically grounded stochastic perturbation continuously ensures the sampler explores high-probability neighborhoods of the true trajectory while remaining tightly guided by the projected measurement constraints $\hat{\bm{z}}_{0|t}$.

\subsection{Initialization Strategy}
\label{sec:init}
Finally, for this hybrid framework to succeed, the initial decoded anchor must reside reasonably close to the natural image manifold. At $t \to 1$ (zero SNR), the decoder outputs chaotic, uninformative geometries, rendering early gradient projections essentially meaningless. Consequently, we instantiate the sampling trajectory at an environmentally stable intermediate threshold $t_0 < 1$ (typically $t_0=0.8$) utilizing an SNR-aware mixed-noise warm-start: \begin{equation}
  \bm{z}_{t_0} = \sigma_{t_0}\,\bm{\epsilon} + (1 - \sigma_{t_0})\,E(\mathcal{A}^\dagger \bm{y}), \quad \bm{\epsilon} \sim \N(\bm{0}, \bm{I}),
  \label{eq:init}
\end{equation}
where $\mathcal{A}^\dagger$ is an inexpensive pseudo-inverse.
This formulation formally bypasses the uninformative zero-SNR regime, firmly anchoring the initial pixel refinement within a valid geometric neighborhood. We provide an analysis validating the effectiveness of this initialization strategy in Section~\ref{sec:exp_Init}. The complete procedure is detailed in Algorithm~\ref{alg:method}.

\subsection{Computational Analysis}
\label{sec:cost}


The per-step cost comprises: (i)~two CFG forward passes of $v_\theta$ (shared with all flow-based baselines); (ii)~$N_x$ pixel-space Langevin steps ($x$-optimization) requiring no neural-network backpropagation; and (iii)~$N_z$ latent alignment gradient steps ($z$-optimization) through $D$ alone.
Traditional latent optimization (e.g., FlowDPS) backpropagates through the composite operator $\mathcal{A} \circ D$. In contrast, our decoupling isolates the measurement physics from the latent prior, transforming complex gradients into two simpler, specialized operations. Our empirical analysis (detailed in Appendix~\ref{app:runtime_analysis}) confirms that the combined execution time of these specialized $x$ and $z$ steps is actually less than evaluating the heavy composite gradient.

\begin{figure*}[t]
    \centering
    \includegraphics[width=0.9\linewidth]{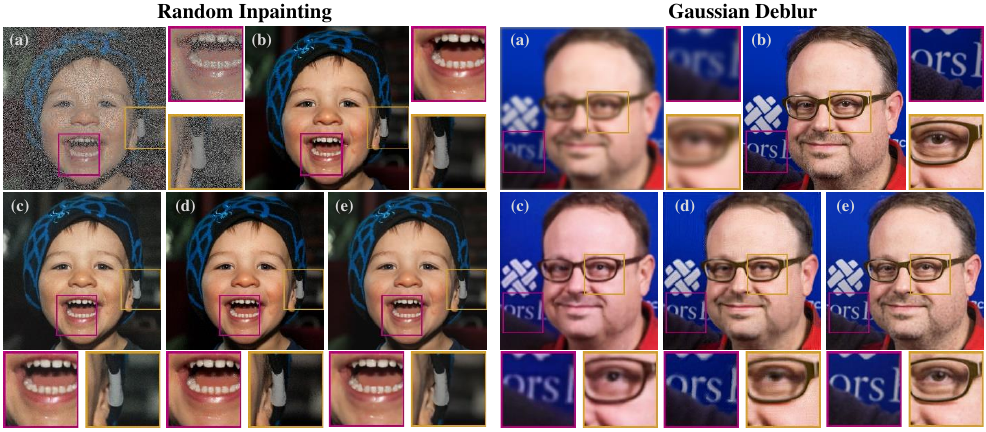}
    \vspace{-0.2cm}
    \caption{Visual ablation. (a)~Measurement. (b)~Reference. (c)~Latent-Only. (d)~Pixel + Encoder. (e)~Ours.} 
    \label{fig:ablation_on_decoupling}
        \vspace{-0.2cm}
\end{figure*}

\section{Experiments}
\label{sec:experiments}

\subsection{Setup}
\label{sec:setup_exp}

\paragraph{Tasks and datasets.}
We evaluate the proposed method on five inverse problems: (i) Gaussian deblurring (kernel size $61$, standard deviation $3.0$), (ii) motion deblurring (kernel size $61$, intensity $0.5$), (iii) $12\times$ super-resolution using bicubic interpolation, (iv) $12\times$ super-resolution using average pooling, and (v) random inpainting with 30\%-70\% pixels masked. To simulate realistic conditions, all measurements are corrupted by additive Gaussian noise with $\sigma_n = 0.03$.

All experiments are conducted on high-resolution images from three widely-used benchmarks, including 1,000 images from the FFHQ~\cite{karras2019style} and AFHQ~\cite{choi2020stargan} validation datasets, respectively, and 800 images from the DIV2K training dataset~\cite{agustsson2017ntire}. We resize all images to a uniform resolution of $768 \times 768$. 

\paragraph{Baselines and metrics.}
We benchmark the proposed HDPS framework against a representative set of state-of-the-art Flow-based solvers: ReSample~\cite{song2024resample}, FlowChef~\cite{patel2025flowchef}, FlowDPS~\cite{kim2025flowdps}, and FLAIR~\cite{erbach2025FLAIR}. We also include LatentDAPS~\cite{zhang2025improving}, which performs posterior sampling directly in the latent space.
To ensure a fair comparison, all solvers are implemented using the pre-trained Stable Diffusion 3 Medium~\cite{esser2024scaling} as the backbone.
For text conditioning, we use dataset-specific prompts (``a photo of a closed face'' for FFHQ, ``a photo of a closed face of a dog/cat'' for AFHQ, and DAPE~\cite{wu2024seesr} captions for DIV2K), with CFG scale 2.0.

We employ a dual assessment strategy: Peak Signal-to-Noise Ratio (PSNR) and Structural Similarity Index (SSIM)~\cite{wang2004image} quantify pixel-level and structural fidelity, while Learned Perceptual Image Patch Similarity (LPIPS)~\cite{zhang2018unreasonable} measures perceptual quality and naturalness.


\subsection{Main Results}
\label{sec:main_results}

\paragraph{Quantitative comparison.}
Table~\ref{tab:quantitative} summarizes results across all datasets and tasks.
HDPS consistently achieves state-of-the-art performance, with the largest gains in tasks involving severe information loss.
On FFHQ, HDPS surpasses the runner-up (FLAIR) by \textbf{+1.88\,dB} in inpainting and \textbf{+2.27\,dB} in motion deblurring, while reducing LPIPS by up to 30\%.
Similar margins are observed on AFHQ (+1.37\,dB / +1.72\,dB).
These gains directly validate our theoretical premise: pixel-space correction recovers high-frequency residuals that latent-only methods discard due to manifold blindness.

In the highly ill-posed $12\times$ super-resolution regime, HDPS remains competitive, achieving the best or second-best scores across all metrics.
Notably, on the general-domain DIV2K dataset, HDPS attains the best distortion--perception balance across all tasks, confirming robustness beyond domain-specific distributions.

\paragraph{Visual comparison.}
Figure~\ref{fig:qualitative} corroborates these quantitative findings.
In $12\times$ SR (Row 2), baselines such as FlowDPS and FlowChef produce broken geometries and aliasing on fine structures (e.g., Ferris wheel struts), whereas HDPS recovers these high-frequency elements faithfully.
In inpainting and deblurring (Rows 1, 3), latent-only methods exhibit over-smoothing (``plastic'' surfaces), while ReSample introduces grid-like artifacts.
HDPS bridges this gap, restoring realistic textures---individual fur strands, skin pores---while maintaining global semantic coherence.

\subsection{Ablation: Decoupling and Alignment Strategy}
\label{sec:exp_ablation_decoupling}
The main contributions of this paper are the decoupled optimization framework and the latent alignment strategy. To validate the effect of each component,
we compare three inference strategies on FFHQ:
    \textbf{(1) Latent-Only}: This variant minimizes the measurement consistency loss via back-propagation through $\mathcal{A} \circ D$;
    \textbf{(2) Pixel + Encoder}: Pixel-space Langevin correction followed by $\bm{z} = E(\hat{\bm{x}})$;
    \textbf{(3) HDPS (Pixel + Alignment)}: Pixel-space Langevin correction followed by decoder inversion (Eq.~\ref{eq:alignment}).
Table~\ref{tab:ablation_on_decoupling} shows clear improvements at each stage.
\textbf{Latent-Only} suffers from stagnation due to manifold blindness, yielding the worst scores.
\textbf{Pixel + Encoder} improves LPIPS but introduces semantic drift---the encoder maps pixel-space artifacts to incorrect latent features, degrading PSNR.
\textbf{Our full HDPS} (Pixel + Alignment) achieves \textbf{+3.38\,dB} over the Encoder variant on inpainting, confirming that optimization-based projection is essential for robustly translating pixel-space corrections into valid latent codes.


Visual results in Figure~\ref{fig:ablation_on_decoupling} demonstrate clear qualitative differences: \textbf{Latent-Only} optimization exhibits noticeable color distortions and spectral instability because back-propagation through the nonlinear decoder amplifies curvature noise. \textbf{Pixel + Encoder} recovers high-frequency details but suffers from semantic drift, as the pre-trained encoder maps gradient artifacts to incorrect latent features. In contrast, \textbf{HDPS} achieves the highest fidelity. The optimization-based projection acts as a robust \emph{manifold filter}, discarding off-manifold Langevin noise while retaining the necessary structural corrections.


\begin{table}[t]
  \centering
  \caption{Ablation study of decoupling and initialization strategies on FFHQ.}
  \vspace{-0.2cm}
  \label{tab:ablation_on_decoupling}
  \resizebox{\linewidth}{!}{
  \begin{tabular}{lcccccc}
    \toprule
     & \multicolumn{3}{c}{Inpainting} & \multicolumn{3}{c}{Gaussian Deblur}\\ 
    Variant  & PSNR$\uparrow$ & SSIM$\uparrow$ & LPIPS$\downarrow$ & PSNR$\uparrow$ & SSIM$\uparrow$ & LPIPS$\downarrow$ \\
    \midrule
    Latent-Only        & 27.19 & 0.683 & 0.101 & 25.98 & 0.781 & 0.207 \\
    Pixel + Encoder    & 30.86 & 0.745 & 0.036 & 27.04 & 0.619 & 0.152 \\
    HDPS w/o Init     & 33.85 & 0.912 & 0.021 & 29.68 & 0.796 & 0.068 \\
    HDPS (Full) & \textbf{34.24} & \textbf{0.922} & \textbf{0.013} & \textbf{30.15} & \textbf{0.823} & \textbf{0.065} \\
    \bottomrule
  \end{tabular}}
      \vspace{-0.2cm}
\end{table}

\subsection{Ablation: Initialization}
\label{sec:exp_Init}

We investigate the initialization time $t_0$ with a fixed budget of NFE\,=\,50. Visual inspection of the intermediate decoded trajectories $D(\bm{z}_{0|t})$ in Figure~\ref{fig:ablation_on_t0} reveals why initialization matters.
At extreme noise levels ($t > 0.8$), the decoded image $D(\bm{z}_{0|t})$ lacks sufficient structural coherence, rendering pixel-space posterior updates uninformative and prone to introducing high-frequency artifacts.
Initializing at $t_0=0.8$ via the SNR-aware strategy effectively circumvents this chaotic regime, ensuring that the semantic guidance from the latent flow prior is geometrically grounded from the very first step.

Table~\ref{tab:ablation_on_decoupling} quantifies this benefit.
The SNR-aware initialization strategy (HDPS full, $t_0=0.8$) consistently outperforms standard Gaussian initialization (HDPS w/o Init, $t_0=1.0$), yielding a performance gain of roughly \textbf{+0.4 dB} in PSNR.

\paragraph{Parameter Sweep.}
To determine the optimal operating point, we performed a parameter sweep varying $t_0 \in \{1.0, 0.9, \dots, 0.5\}$. 
We find that $t_0=0.8$ achieves the optimum balance between structural fidelity and perceptual quality. 
A more detailed quantitative analysis and discussion on the impact of $t_0$ are deferred to Appendix~\ref{app:init_sweep}.


\begin{figure}[t]
    \centering
    \includegraphics[width=\linewidth]{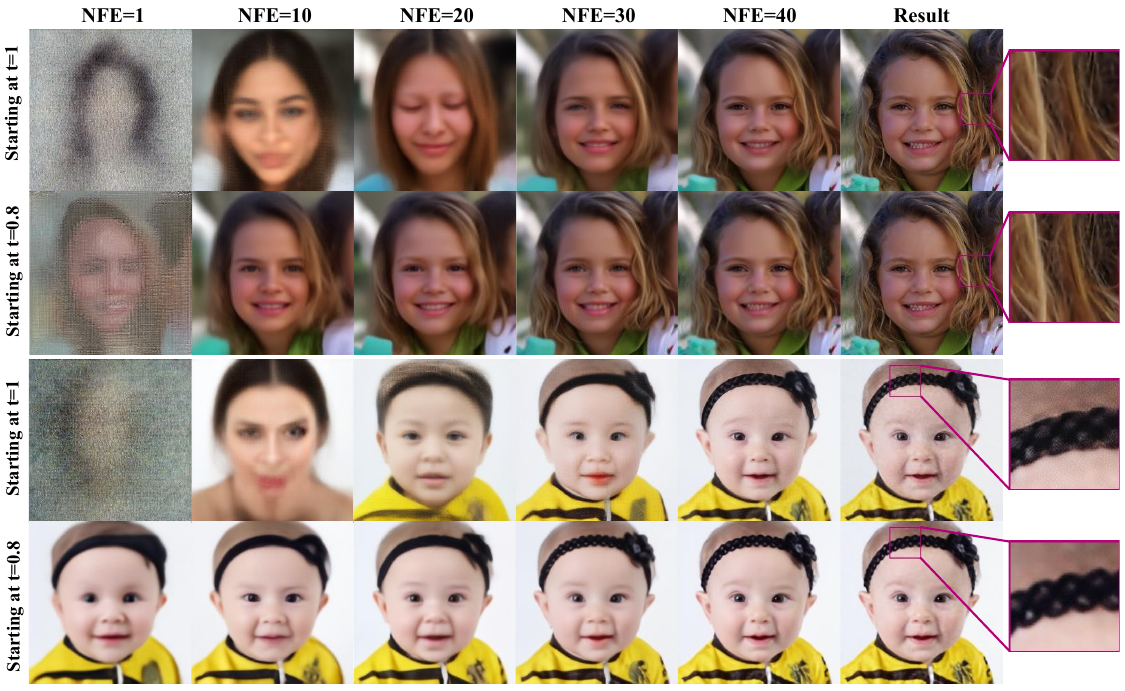}
    \caption{Trajectory evolution. Starting at $t_0=1.0$ produces chaotic early structures that persist as artifacts; warm-start at $t_0=0.8$ yields coherent refinement from the first step.}
    \label{fig:ablation_on_t0}
    \vspace{-0.2cm}
\end{figure}

\subsection{Runtime Analysis}
\label{sec:runtime}

In Table~\ref{tab:runtime}, we compare inference time and reconstruction quality across methods under a fixed NFE budget on FFHQ Gaussian deblur, analyzing both \textcolor{blue}{low-budget} ($N_z=3$) and \textcolor{red}{high-budget} ($N_z=15$) regimes.

Two key empirical observations emerge:
First, \textbf{HDPS ($N_z\!=\!3$) surpasses FlowDPS ($N_z\!=\!15$) by 2.26\,dB} while using $5\times$ fewer latent gradient steps. This confirms that resolving manifold blindness via exact pixel-space physics is fundamentally more effective than blindly scaling iteration counts in the restricted latent space.
Second, at equal $N_z$, HDPS is \textbf{faster} than FlowDPS (e.g., 6.21\,s vs.\ 6.86\,s at $N_z\!=\!3$). FlowDPS must evaluate gradients through the composite operator $\mathcal{A}\circ D$. By decoupling measurement physics from manifold projection, HDPS transforms the complex backward pass into two specialized sub-tasks, thus yielding faster wall-clock execution. We provide further runtime analysis in Appendix~\ref{app:runtime_analysis}.


\subsection{Hyperparameter Sensitivity}
\label{sec:hyperparam}

We analyze the sensitivity of HDPS to the inner-loop iterations: Langevin refinement steps $N_x$ and latent alignment steps $N_z$.
As shown in Figure~\ref{fig:ablation_Nx_Nz}(a), increasing $N_x$ initially improves perceptual quality (lower LPIPS) by recovering high-frequency textures lost by the decoder. However, excessively large $N_x$ ($> 20$) leads to over-sharpening and noise accumulation, degrading PSNR.
Figure~\ref{fig:ablation_Nx_Nz}(b) reveals a similar convex trend for $N_z$. Increasing alignment steps initially improves both metrics by ensuring the latent code accurately reflects pixel-space corrections. Beyond $N_z=15$, the optimization begins to overfit the latent code to the noisy intermediate $\hat{\bm{x}}$, causing performance to decline.
Based on optimal trade-offs, we adopt $N_x = 20$ and $N_z = 15$ as default configurations.

\begin{table}[t]
  \centering
  \caption{Runtime and quality (NFE\,=\,50, Gaussian Deblurring, FFHQ).
  \textcolor{blue}{Blue} / \textcolor{red}{red} = $N\!=\!3$ / $N\!=\!15$. Time on a single RTX 4090.}
  \label{tab:runtime}
  \resizebox{\linewidth}{!}{
  \begin{tabular}{lrrrrrrrr}
    \toprule
    & \textbf{ReSample} & \textbf{FlowChef} & \textbf{FLAIR}
    & \textcolor{blue}{\textbf{FlowDPS}} & \textcolor{blue}{\textbf{HDPS}}
    & \textcolor{red}{\textbf{FlowDPS}} & \textcolor{red}{\textbf{HDPS}} \\
    \midrule
    NFE                           & 50    & 100   & 50    & 50   & 50   & 50    & 50   \\
    Iters ($N_z$, $+N_x$) & 30 & 1 & 15 & \textcolor{blue}{3} & \textcolor{blue}{3\,(+3)} & \textcolor{red}{15} & \textcolor{red}{15\,(+15)} \\
    Time (s)            & 25.25 & 10.27 & 22.13 & \textcolor{blue}{6.86}  & \textcolor{blue}{6.21}  & \textcolor{red}{16.95} & \textcolor{red}{15.46} \\
    PSNR (dB)                     & 22.47 & 25.14 & 28.33 & \textcolor{blue}{26.59} & \textcolor{blue}{28.82} & \textcolor{red}{26.56} & \textcolor{red}{30.18} \\
    \bottomrule
  \end{tabular}}
\end{table}

\begin{figure}[t]
    \centering
    \captionsetup[subfigure]{skip=4pt}
    
    \begin{subfigure}{1.0\linewidth}
        \centering
        \includegraphics[width=0.49\linewidth]{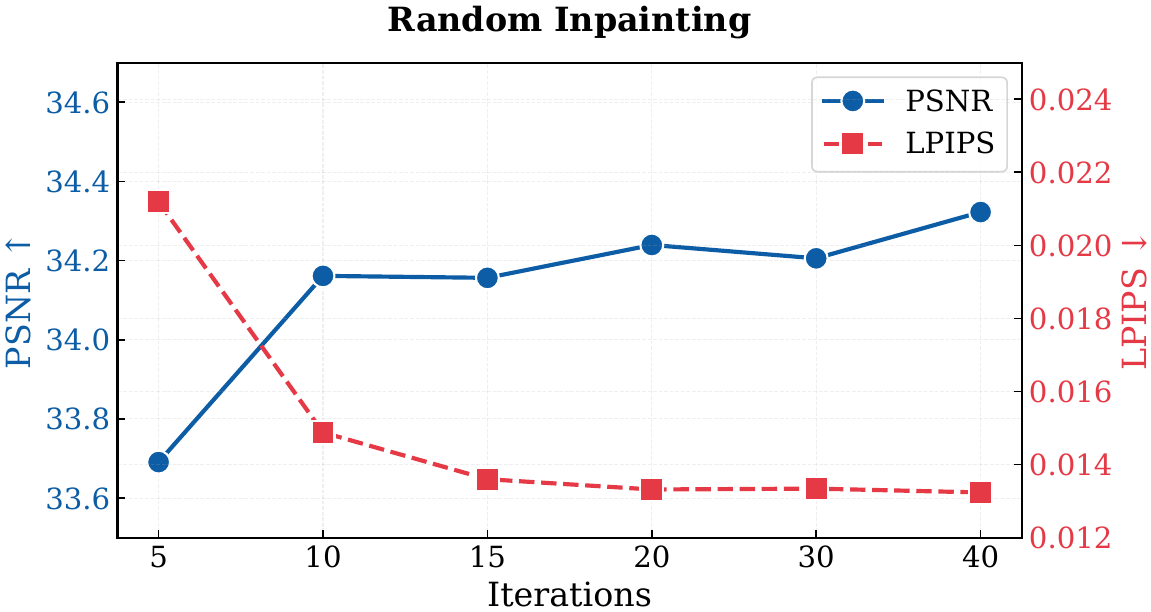}
        \hfill
        \includegraphics[width=0.49\linewidth]{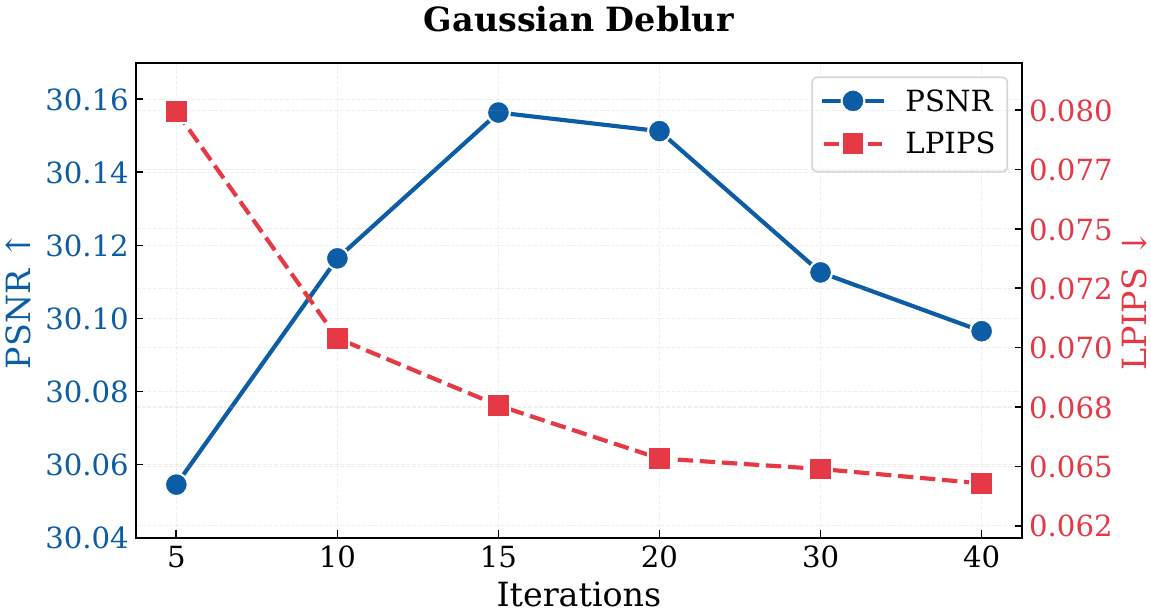}
        \vspace{-2pt}
        \caption*{(a) Effect of pixel-space iterations $N_x$}
    \end{subfigure}

    \vspace{4pt} 

    \begin{subfigure}{1.0\linewidth}
        \centering
        \includegraphics[width=0.49\linewidth]{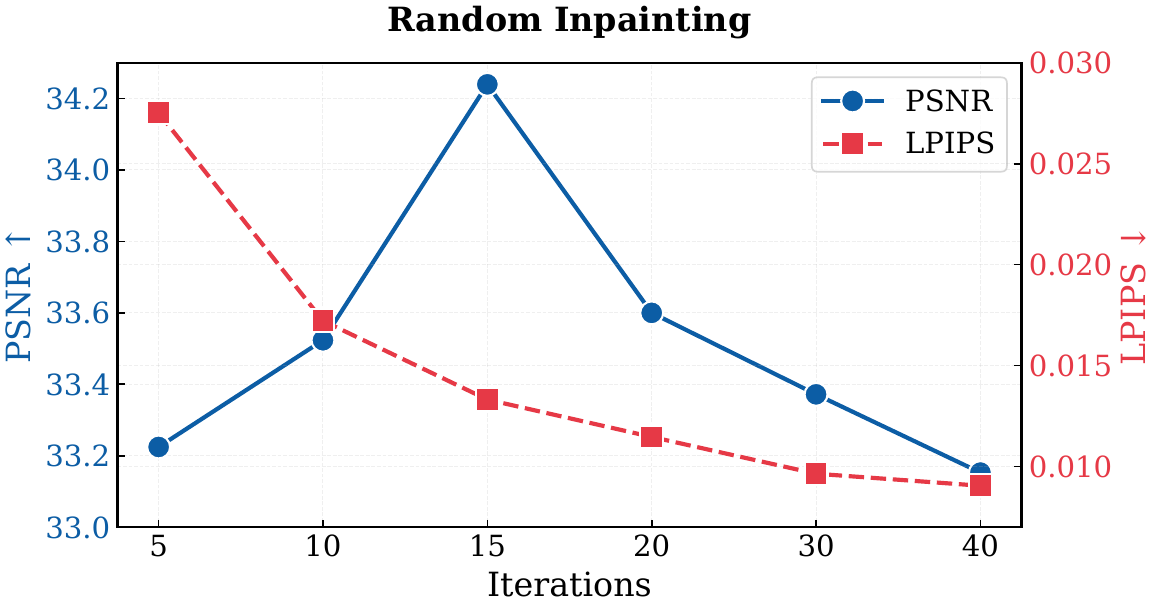}
        \hfill
        \includegraphics[width=0.49\linewidth]{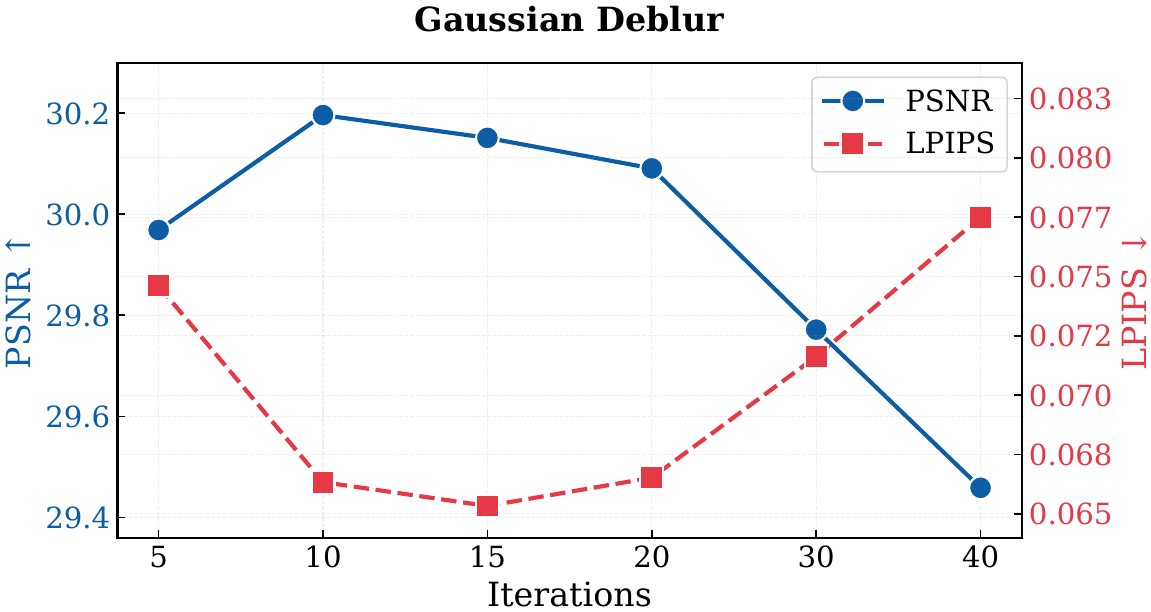}
        \vspace{-2pt}
        \caption*{(b) Effect of latent-space iterations $N_z$}
    \end{subfigure}

    \vspace{2pt}
    \caption{\textbf{Sensitivity to inner-loop iterations.} We evaluate the impact of (a) $N_x$ and (b) $N_z$ across different tasks.}
    \label{fig:ablation_Nx_Nz}
    \vspace{-0.2cm}
\end{figure}

\section{Conclusion}
\label{sec:conclusion}

We identified \emph{First-Order Manifold Blindness}---a fundamental geometric bottleneck caused by the rank deficiency of the decoder Jacobian---that limits latent-only inverse solvers from recovering high-frequency measurement residuals.
To resolve this, we proposed {HDPS}, a hybrid-domain framework that assigns measurement consistency to the pixel space and semantic prior modeling to the latent space, connected by a robust optimization-based latent alignment.
Extensive experiments demonstrate state-of-the-art performance across five linear inverse problems on three datasets, with gains exceeding 2\,dB over the strongest baselines.

\paragraph{Limitations and future work.}
The alternating pixel--latent optimization introduces computational overhead from the inner loops ($N_x$, $N_z$), though we show this cost is comparable to existing methods at matched quality.
The method's ceiling is ultimately bounded by the decoder's representational capacity.
Future directions include integrating faster pixel-space solvers (e.g., consistency models) and extending the decoupled paradigm to blind inverse problems where the forward operator $\mathcal{A}$ is unknown.

\begin{acks}
This work is supported by the National Major Scientific Instruments and Equipments Development Project of National Natural Science Foundation of China under Grant 62427820, the Science Fund for Creative Research Groups of Sichuan Province Natural Science Foundation under Grant 2024NSFTD0035, the Ministry of Education Engineering Research Center Guiding Project for Machine Learning and Industrial Intelligence Applications under Grant SCU2024D013.
\end{acks}

\bibliographystyle{ACM-Reference-Format}
\balance
\bibliography{ref}

@String{Computer = "{IEEE} Computer" }

@article{ho2020denoising,
  title={Denoising diffusion probabilistic models},
  author={Ho, Jonathan and Jain, Ajay and Abbeel, Pieter},
  journal={Advances in Neural Information Processing Systems},
  volume={33},
  pages={6840--6851},
  year={2020}
}

@inproceedings{
  song2021scorebased,
  title={Score-Based Generative Modeling through Stochastic Differential Equations},
  author={Yang Song and Jascha Sohl-Dickstein and Diederik P Kingma and Abhishek Kumar and Stefano Ermon and Ben Poole},
  booktitle={International Conference on Learning Representations},
  year={2021},
  url={https://openreview.net/forum?id=PxTIG12RRHS}
}

@inproceedings{rombach2022high,
  title={High-resolution image synthesis with latent diffusion models},
  author={Rombach, Robin and Blattmann, Andreas and Lorenz, Dominik and Esser, Patrick and Ommer, Bj{\"o}rn},
  booktitle={Proceedings of the IEEE/CVF Conference on Computer Vision and Pattern Recognition},
  pages={10684--10695},
  year={2022}
}

@inproceedings{
ho2021classifierfree,
title={Classifier-Free Diffusion Guidance},
author={Jonathan Ho and Tim Salimans},
booktitle={NeurIPS 2021 Workshop on Deep Generative Models and Downstream Applications},
year={2021},
url={https://openreview.net/forum?id=qw8AKxfYbI}
}

@inproceedings{
kawar2022denoising,
title={Denoising Diffusion Restoration Models},
author={Bahjat Kawar and Michael Elad and Stefano Ermon and Jiaming Song},
booktitle={Advances in Neural Information Processing Systems},
editor={Alice H. Oh and Alekh Agarwal and Danielle Belgrave and Kyunghyun Cho},
year={2022},
url={https://openreview.net/forum?id=kxXvopt9pWK}
}

@inproceedings{
chung2022improving,
title={Improving Diffusion Models for Inverse Problems using Manifold Constraints},
author={Hyungjin Chung and Byeongsu Sim and Dohoon Ryu and Jong Chul Ye},
booktitle={Advances in Neural Information Processing Systems},
editor={Alice H. Oh and Alekh Agarwal and Danielle Belgrave and Kyunghyun Cho},
year={2022},
url={https://openreview.net/forum?id=nJJjv0JDJju}
}

@inproceedings{
chung2023diffusion,
title={Diffusion Posterior Sampling for General Noisy Inverse Problems},
author={Hyungjin Chung and Jeongsol Kim and Michael Thompson Mccann and Marc Louis Klasky and Jong Chul Ye},
booktitle={International Conference on Learning Representations},
year={2023},
url={https://openreview.net/forum?id=OnD9zGAGT0k}
}

@inproceedings{zhang2018unreasonable,
  title={The unreasonable effectiveness of deep features as a perceptual metric},
  author={Zhang, Richard and Isola, Phillip and Efros, Alexei A and Shechtman, Eli and Wang, Oliver},
  booktitle={Proceedings of the IEEE conference on computer vision and pattern recognition},
  pages={586--595},
  year={2018}
}

@article{kingma2013auto,
  title={Auto-encoding variational bayes},
  author={Kingma, Diederik P and Welling, Max},
  journal={arXiv preprint arXiv:1312.6114},
  year={2013}
}

@inproceedings{
wang2023zeroshot,
title={Zero-Shot Image Restoration Using Denoising Diffusion Null-Space Model},
author={Yinhuai Wang and Jiwen Yu and Jian Zhang},
booktitle={The Eleventh International Conference on Learning Representations },
year={2023},
url={https://openreview.net/forum?id=mRieQgMtNTQ}
}

@inproceedings{
song2023pseudoinverseguided,
title={Pseudoinverse-Guided Diffusion Models for Inverse Problems},
author={Jiaming Song and Arash Vahdat and Morteza Mardani and Jan Kautz},
booktitle={International Conference on Learning Representations},
year={2023},
url={https://openreview.net/forum?id=9_gsMA8MRKQ}
}

@inproceedings{rout2023solving,
  title={Solving Linear Inverse Problems Provably via Posterior Sampling with Latent Diffusion Models},
  author={Rout, Litu and Raoof, Negin and Daras, Giannis and Caramanis, Constantine and Dimakis, Alex and Shakkottai, Sanjay},
  booktitle={Thirty-seventh Conference on Neural Information Processing Systems},
  year={2023}
}

@inproceedings{nichol2021improved,
  title={Improved denoising diffusion probabilistic models},
  author={Nichol, Alexander Quinn and Dhariwal, Prafulla},
  booktitle={International Conference on Machine Learning},
  pages={8162--8171},
  year={2021},
  organization={PMLR}
}

@inproceedings{song2023loss,
  title={Loss-guided diffusion models for plug-and-play controllable generation},
  author={Song, Jiaming and Zhang, Qinsheng and Yin, Hongxu and Mardani, Morteza and Liu, Ming-Yu and Kautz, Jan and Chen, Yongxin and Vahdat, Arash},
  booktitle={International Conference on Machine Learning},
  pages={32483--32498},
  year={2023},
  organization={PMLR}
}

@article{he2023iterative,
  title={Iterative reconstruction based on latent diffusion model for sparse data reconstruction},
  author={He, Linchao and Yan, Hongyu and Luo, Mengting and Luo, Kunming and Wang, Wang and Du, Wenchao and Chen, Hu and Yang, Hongyu and Zhang, Yi},
  journal={arXiv preprint arXiv:2307.12070},
  year={2023}
}

@inproceedings{zhu2023denoising,
  title={Denoising Diffusion Models for Plug-and-Play Image Restoration},
  author={Zhu, Yuanzhi and Zhang, Kai and Liang, Jingyun and Cao, Jiezhang and Wen, Bihan and Timofte, Radu and Van Gool, Luc},
  booktitle={Proceedings of the IEEE/CVF Conference on Computer Vision and Pattern Recognition},
  pages={1219--1229},
  year={2023}
}

@inproceedings{chen2021equivariant,
  title={Equivariant imaging: Learning beyond the range space},
  author={Chen, Dongdong and Tachella, Juli{\'a}n and Davies, Mike E},
  booktitle={Proceedings of the IEEE/CVF International Conference on Computer Vision},
  pages={4379--4388},
  year={2021}
}

@article{vahdat2021score,
  title={Score-based generative modeling in latent space},
  author={Vahdat, Arash and Kreis, Karsten and Kautz, Jan},
  journal={Advances in Neural Information Processing Systems},
  volume={34},
  pages={11287--11302},
  year={2021}
}

@inproceedings{
chung2024prompttuning,
title={Prompt-tuning Latent Diffusion Models for Inverse Problems},
author={Hyungjin Chung and Jong Chul Ye and Peyman Milanfar and Mauricio Delbracio},
booktitle={Forty-first International Conference on Machine Learning},
year={2024},
url={https://openreview.net/forum?id=hrwIndai8e}
}

@inproceedings{rout2024beyond,
  title={Beyond first-order tweedie: Solving inverse problems using latent diffusion},
  author={Rout, Litu and Chen, Yujia and Kumar, Abhishek and Caramanis, Constantine and Shakkottai, Sanjay and Chu, Wen-Sheng},
  booktitle={Proceedings of the IEEE/CVF Conference on Computer Vision and Pattern Recognition},
  pages={9472--9481},
  year={2024}
}

@inproceedings{
song2024resample,
title={Solving Inverse Problems with Latent Diffusion Models via Hard Data Consistency},
author={Bowen Song and Soo Min Kwon and Zecheng Zhang and Xinyu Hu and Qing Qu and Liyue Shen},
booktitle={The Twelfth International Conference on Learning Representations},
year={2024},
url={https://openreview.net/forum?id=j8hdRqOUhN}
}

@inproceedings{esser2024scaling,
  title={Scaling rectified flow transformers for high-resolution image synthesis},
  author={Esser, Patrick and Kulal, Sumith and Blattmann, Andreas and Entezari, Rahim and M{\"u}ller, Jonas and Saini, Harry and Levi, Yam and Lorenz, Dominik and Sauer, Axel and Boesel, Frederic and others},
  booktitle={Forty-first International Conference on Machine Learning},
  year={2024}
}

@inproceedings{
chung2024decomposed,
title={Decomposed Diffusion Sampler for Accelerating Large-Scale Inverse Problems},
author={Hyungjin Chung and Suhyeon Lee and Jong Chul Ye},
booktitle={The Twelfth International Conference on Learning Representations},
year={2024},
url={https://openreview.net/forum?id=DsEhqQtfAG}
}

@inproceedings{
liu2023flow,
title={Flow Straight and Fast: Learning to Generate and Transfer Data with Rectified Flow},
author={Xingchao Liu and Chengyue Gong and qiang liu},
booktitle={The Eleventh International Conference on Learning Representations },
year={2023},
url={https://openreview.net/forum?id=XVjTT1nw5z}
}

@inproceedings{wu2024diffusion,
  title={Diffusion Posterior Proximal Sampling for Image Restoration},
  author={Wu, Hongjie and He, Linchao and Zhang, Mingqin and Chen, Dongdong and Luo, Kunming and Luo, Mengting and Zhou, Ji-Zhe and Chen, Hu and Lv, Jiancheng},
  booktitle={Proceedings of the 32nd ACM International Conference on Multimedia},
  pages={214--223},
  year={2024}
}

@inproceedings{karras2019style,
  title={A style-based generator architecture for generative adversarial networks},
  author={Karras, Tero and Laine, Samuli and Aila, Timo},
  booktitle={Proceedings of the IEEE/CVF conference on computer vision and pattern recognition},
  pages={4401--4410},
  year={2019}
}

@inproceedings{zhang2025improving,
  title={Improving diffusion inverse problem solving with decoupled noise annealing},
  author={Zhang, Bingliang and Chu, Wenda and Berner, Julius and Meng, Chenlin and Anandkumar, Anima and Song, Yang},
  booktitle={Proceedings of the Computer Vision and Pattern Recognition Conference},
  pages={20895--20905},
  year={2025}
}

@inproceedings{
lipman2023flow,
title={Flow Matching for Generative Modeling},
author={Yaron Lipman and Ricky T. Q. Chen and Heli Ben-Hamu and Maximilian Nickel and Matthew Le},
booktitle={The Eleventh International Conference on Learning Representations },
year={2023},
url={https://openreview.net/forum?id=PqvMRDCJT9t}
}

@article{wu2024principled,
  title={Principled Probabilistic Imaging using Diffusion Models as Plug-and-Play Priors},
  author={Wu, Zihui and Sun, Yu and Chen, Yifan and Zhang, Bingliang and Yue, Yisong and Bouman, Katherine L},
  journal={arXiv e-prints},
  pages={arXiv--2405},
  year={2024}
}

@article{ho2022video,
  title={Video diffusion models},
  author={Ho, Jonathan and Salimans, Tim and Gritsenko, Alexey and Chan, William and Norouzi, Mohammad and Fleet, David J},
  journal={Advances in Neural Information Processing Systems},
  volume={35},
  pages={8633--8646},
  year={2022}
}

@inproceedings{wu2025enhancing,
  title={Enhancing Diffusion Model Stability for Image Restoration via Gradient Management},
  author={Wu, Hongjie and Zhang, Mingqin and He, Linchao and Zhou, Ji-Zhe and Lv, Jiancheng},
  booktitle={Proceedings of the 33rd ACM International Conference on Multimedia},
  pages={10768--10777},
  year={2025}
}

@article{sd35,
  title={SD3. 5-Flash: Distribution-Guided Distillation of Generative Flows},
  author={Bandyopadhyay, Hmrishav and Entezari, Rahim and Scott, Jim and Adithyan, Reshinth and Song, Yi-Zhe and Jampani, Varun},
  journal={arXiv preprint arXiv:2509.21318},
  year={2025}
}

@article{labs2025flux,
  title={FLUX. 1 Kontext: Flow Matching for In-Context Image Generation and Editing in Latent Space},
  author={Labs, Black Forest and Batifol, Stephen and Blattmann, Andreas and Boesel, Frederic and Consul, Saksham and Diagne, Cyril and Dockhorn, Tim and English, Jack and English, Zion and Esser, Patrick and others},
  journal={arXiv preprint arXiv:2506.15742},
  year={2025}
}

@inproceedings{kim2025flowdps,
  title={Flowdps: Flow-driven posterior sampling for inverse problems},
  author={Kim, Jeongsol and Kim, Bryan Sangwoo and Ye, Jong Chul},
  booktitle={Proceedings of the IEEE/CVF International Conference on Computer Vision},
  pages={12328--12337},
  year={2025}
}

@article{park2025flowlps,
  title={FlowLPS: Langevin-Proximal Sampling for Flow-based Inverse Problem Solvers},
  author={Park, Jonghyun and Ye, Jong Chul},
  journal={arXiv preprint arXiv:2512.07150},
  year={2025}
}

@article{erbach2025FLAIR,
  title={Solving Inverse Problems with FLAIR},
  author={Erbach, Julius and Narnhofer, Dominik and Dombos, Andreas and Schiele, Bernt and Lenssen, Jan Eric and Schindler, Konrad},
  journal={arXiv preprint arXiv:2506.02680},
  year={2025}
}

@inproceedings{
martin2025pnpflow,
title={PnP-Flow: Plug-and-Play Image Restoration with Flow Matching},
author={S{\'e}gol{\`e}ne Tiffany Martin and Anne Gagneux and Paul Hagemann and Gabriele Steidl},
booktitle={The Thirteenth International Conference on Learning Representations},
year={2025},
url={https://openreview.net/forum?id=5AtHrq3B5R}
}

@article{pourya2025flower,
  title={FLOWER: A Flow-Matching Solver for Inverse Problems},
  author={Pourya, Mehrsa and Rawas, Bassam El and Unser, Michael},
  journal={arXiv preprint arXiv:2509.26287},
  year={2025}
}

@article{askari2025latent,
  title={Latent Refinement via Flow Matching for Training-free Linear Inverse Problem Solving},
  author={Askari, Hossein and Luo, Yadan and Sun, Hongfu and Roosta, Fred},
  journal={arXiv preprint arXiv:2511.06138},
  year={2025}
}

@inproceedings{yan2025fig,
  title={Fig: Flow with interpolant guidance for linear inverse problems},
  author={Yan, Yici and Zhang, Yichi and Meng, Xiangming and Zhao, Zhizhen},
  booktitle={The Thirteenth International Conference on Learning Representations},
  year={2025}
}

@article{zhang2024flow,
  title={Flow priors for linear inverse problems via iterative corrupted trajectory matching},
  author={Zhang, Yasi and Yu, Peiyu and Zhu, Yaxuan and Chang, Yingshan and Gao, Feng and Wu, Ying N and Leong, Oscar},
  journal={Advances in Neural Information Processing Systems},
  volume={37},
  pages={57389--57417},
  year={2024}
}

@article{ben2024d,
  title={D-flow: Differentiating through flows for controlled generation},
  author={Ben-Hamu, Heli and Puny, Omri and Gat, Itai and Karrer, Brian and Singer, Uriel and Lipman, Yaron},
  journal={arXiv preprint arXiv:2402.14017},
  year={2024}
}

@inproceedings{zhang2023adding,
  title={Adding conditional control to text-to-image diffusion models},
  author={Zhang, Lvmin and Rao, Anyi and Agrawala, Maneesh},
  booktitle={Proceedings of the IEEE/CVF international conference on computer vision},
  pages={3836--3847},
  year={2023}
}

@inproceedings{
zilberstein2025repulsive,
title={Repulsive Latent Score Distillation for Solving Inverse Problems},
author={Nicolas Zilberstein and Morteza Mardani and Santiago Segarra},
booktitle={The Thirteenth International Conference on Learning Representations},
year={2025},
url={https://openreview.net/forum?id=bwJxUB0y46}
}

@inproceedings{zhang2025decoupling,
  title={Decoupling training-free guided diffusion by admm},
  author={Zhang, Youyuan and Liu, Zehua and Li, Zenan and Li, Zhaoyu and Clark, James J and Si, Xujie},
  booktitle={Proceedings of the Computer Vision and Pattern Recognition Conference},
  pages={23292--23302},
  year={2025}
}

@inproceedings{
patel2025flowchef,
title={Steering Rectified Flow Models in the Vector Field for Controlled Image Generation},
author={Maitreya Patel and Song Wen and Dimitris N. Metaxas and Yezhou Yang},
booktitle={Frontiers in Probabilistic Inference: Learning meets Sampling},
year={2025},
url={https://openreview.net/forum?id=p3DUUpd2No}
}

@book{tarantola2005inverse,
  title={Inverse problem theory and methods for model parameter estimation},
  author={Tarantola, Albert},
  year={2005},
  publisher={SIAM}
}

@inproceedings{choi2020stargan,
  title={Stargan v2: Diverse image synthesis for multiple domains},
  author={Choi, Yunjey and Uh, Youngjung and Yoo, Jaejun and Ha, Jung-Woo},
  booktitle={Proceedings of the IEEE/CVF conference on computer vision and pattern recognition},
  pages={8188--8197},
  year={2020}
}

@inproceedings{agustsson2017ntire,
  title={Ntire 2017 challenge on single image super-resolution: Dataset and study},
  author={Agustsson, Eirikur and Timofte, Radu},
  booktitle={Proceedings of the IEEE conference on computer vision and pattern recognition workshops},
  pages={126--135},
  year={2017}
}

@inproceedings{wu2024seesr,
  title={Seesr: Towards semantics-aware real-world image super-resolution},
  author={Wu, Rongyuan and Yang, Tao and Sun, Lingchen and Zhang, Zhengqiang and Li, Shuai and Zhang, Lei},
  booktitle={Proceedings of the IEEE/CVF conference on computer vision and pattern recognition},
  pages={25456--25467},
  year={2024}
}

@article{wang2004image,
  title={Image quality assessment: from error visibility to structural similarity},
  author={Wang, Zhou and Bovik, Alan C and Sheikh, Hamid R and Simoncelli, Eero P},
  journal={IEEE transactions on image processing},
  volume={13},
  number={4},
  pages={600--612},
  year={2004},
  publisher={IEEE}
}

\clearpage

\appendix

\section{Proof of Theorem~\ref{thm:resolution}}
\label{app:proof}

We analyze the optimal projection $\hat{\bm{z}}_{0|t}(\eta_x)$ via Taylor expansion around $\eta_x = 0$.
Write $\hat{\bm{z}}_{0|t} = \bm{z}_{0|t} + \eta_x \bm{v} + \frac{1}{2}\eta_x^2 \bm{w} + \mathcal{O}(\eta_x^3)$.
The decoded image expands as:
\begin{equation}
    D(\hat{\bm{z}}_{0|t}) = \bm{x}_{0|t} + \eta_x J_D \bm{v} + \tfrac{1}{2}\eta_x^2 (J_D \bm{w} + \bm{v}^\top \mathcal{H}_D \bm{v}) + \mathcal{O}(\eta_x^3),
\end{equation}
where $J_D$ and $\mathcal{H}_D$ are evaluated at $\bm{z}_{0|t}$.

The first-order optimality condition for the projection $\hat{\bm{z}}_{0|t} = \argmin_{\bm{z}} \|\hat{\bm{x}}_{0|t} - D(\bm{z})\|^2$ is:
\begin{equation}
    J_D(\hat{\bm{z}}_{0|t})^\top \big(D(\hat{\bm{z}}_{0|t}) - \hat{\bm{x}}_{0|t}\big) = \bm{0}.
\end{equation}

\paragraph{First-order matching ($\mathcal{O}(\eta_x)$).}
Substituting $\hat{\bm{x}}_{0|t} = \bm{x}_{0|t} - \eta_x \bm{g}$ and matching first-order terms yields:
\begin{equation}
    J_D^\top (J_D \bm{v} + \bm{g}) = \bm{0}.
\end{equation}
Since $\bm{g} = \bm{g}_{\parallel} + \bm{g}_{\perp}$ and $J_D^\top \bm{g}_{\perp} = \bm{0}$ (by orthogonality), this reduces to:
\begin{equation}
    J_D^\top J_D \bm{v} = -J_D^\top \bm{g}_{\parallel},
\end{equation}
giving $\bm{v} = -(J_D^\top J_D)^{\dagger} J_D^\top \bm{g}_{\parallel}$, where $\dagger$ denotes the Moore-Penrose pseudoinverse.

\paragraph{Projection onto $\bm{g}_{\perp}$.}
The effective image-space update is $\Delta \bm{x} = D(\hat{\bm{z}}_{0|t}) - \bm{x}_{0|t}$.
Projecting onto $\bm{g}_{\perp}$:
\begin{align}
    \langle \Delta \bm{x}, \bm{g}_{\perp} \rangle 
    &= \eta_x \underbrace{\langle J_D \bm{v}, \bm{g}_{\perp} \rangle}_{=\,0} \nonumber \\
    &\quad + \tfrac{1}{2}\eta_x^2 \Big( \underbrace{\langle J_D \bm{w}, \bm{g}_{\perp} \rangle}_{=\,0} + \langle \bm{v}^\top \mathcal{H}_D \bm{v}, \bm{g}_{\perp} \rangle \Big) + \mathcal{O}(\eta_x^3).
\end{align}
The terms $\langle J_D \bm{v}, \bm{g}_{\perp} \rangle$ and $\langle J_D \bm{w}, \bm{g}_{\perp} \rangle$ both vanish because $J_D \bm{v}$ and $J_D \bm{w}$ lie in $\mathcal{R}(J_D)$, while $\bm{g}_{\perp} \in \mathcal{R}(J_D)^\perp$.
The surviving term is:
\begin{equation}
    \langle \Delta \bm{x}, \bm{g}_{\perp} \rangle = \tfrac{1}{2}\eta_x^2 \, \bm{v}^\top \big(\bm{g}_{\perp}^\top \mathcal{H}_D\big) \bm{v} + \mathcal{O}(\eta_x^3).
\end{equation}

This is non-zero whenever $\bm{g}_{\perp}^\top \mathcal{H}_D \neq \bm{0}$, i.e., whenever the decoder has non-trivial curvature along directions correlated with the orthogonal residual.
In contrast, pure latent optimization yields $\langle \Delta \bm{x}, \bm{g}_{\perp} \rangle = 0$ at \emph{all} orders of $\eta$, since latent updates are confined to $\mathcal{R}(J_D)$ by construction.
\hfill $\square$

\section{Implementation Details}
\label{app:details}

\subsection{HDPS Configurations}

\paragraph{Hyperparameters.}
Table~\ref{tab:hyperparams} lists the default hyperparameters used across all experiments unless otherwise noted.

\begin{table}[h]
  \centering
  \caption{Default hyperparameters for HDPS.}
  \label{tab:hyperparams}
  \begin{tabular}{lcc}
    \toprule
    Parameter & Symbol & Value \\
    \midrule
    Number of flow steps (NFE) & -- & 50 \\
    Initialization time & $t_0$ & 0.8 \\
    Langevin steps per flow step & $N_x$ & 20 \\
    Alignment steps per flow step & $N_z$ & 15 \\
    CFG scale & $\lambda$ & 2.0 \\
    Measurement noise level & $\sigma_n$ & 0.03 \\
    \bottomrule
  \end{tabular}
\end{table}

\paragraph{Text prompts.}
For classifier-free guidance conditioning, we use the following dataset-specific prompts:
\begin{itemize}
    \item \textbf{FFHQ}: ``a photo of a closed face''
    \item \textbf{AFHQ}: ``a photo of a closed face of a dog'' or ``a photo of a closed face of a cat''
    \item \textbf{DIV2K}: automated descriptions generated by DAPE~\cite{wu2024seesr}
\end{itemize}








\subsection{Comparison Methods}
To ensure a fair comparison, all baseline methods are implemented using the same pre-trained \textbf{Stable~Diffusion 3.0}~\cite{esser2024scaling} backbone with the guidance scale fixed at 2.0.

\textbf{LatentDAPS \cite{zhang2025improving}:} We use 50 NFEs and the hyperparameter $\beta_y$ is set to $ 1 \times 10^{-4}$. The Langevin Dynamics stage consists of $N=15$ steps with a step size of $\eta = 1 \times 10^{-4}$.

\textbf{ReSample \cite{song2024resample}:} We fix the NFE at 50 with a skip step size of 1, and the total optimization steps at $N=30$. The resampling hyper-parameter $\gamma$ and step size $\eta$ are task-specific: $\{ \gamma=0.5, \eta=10 \}$ for super-resolution and $\{ \gamma=1, \eta=1 \}$ for other tasks.

\textbf{FlowChef \cite{patel2025flowchef}:} The generation process uses 100 NFEs with a constant step size of 0.5 for data consistency.

\textbf{FlowDPS \cite{kim2025flowdps}:} Following its original configuration, we set the NFE to 50 and perform 3 gradient descent steps for data consistency with a step size of 15 across all tasks.

\textbf{FLAIR \cite{erbach2025FLAIR}:} We adopt the regularization weights proposed in the original work. For data consistency, 15 gradient descent steps are employed, with the step size $\eta$ configured as 12 for super-resolution and 0.1 for other tasks.

\section{Additional Experimental Results}
\label{app:experimental}

\subsection{Detailed Runtime Analysis}
\label{app:runtime_analysis}

Table~\ref{tab:time_analysis} provides a detailed breakdown of the time consumed by different optimization components during the ODE integration process. 
We compare the computational cost of the composite gradient step in FlowDPS against the decoupled operations in HDPS.
As shown, FlowDPS spends more time evaluating gradients through the heavy composite operator $\mathcal{A} \circ D$. In contrast, HDPS breaks the process down into two independent and specialized operations: an inexpensive $x$-optimization step (requiring only forward/adjoint evaluations) and a $z$-optimization step (requiring only decoder backward passes). Crucially, the total combined time of these two decoupled steps in HDPS ($x$+$z$ optimization) remains lower than the single composite step in FlowDPS across all iteration budgets ($N \in \{1, 3, 15\}$). This confirms that decoupling fundamentally improves computational scaling by bypassing deep composite back-propagation chains.

\begin{table}[htbp]
\centering
\caption{Runtime of different optimization items with NFE=50.}
\label{tab:time_analysis}
\begin{tabular}{lccc}
\toprule
Optimization Items & 1  & 3  & 15  \\ \midrule
FlowDPS, composite $\mathcal{A}\circ D$ (s) & 0.8472 & 2.4127 & 12.6074 \\
HDPS, $z$ optimization (s)    & 0.6824 & 2.0001 & 10.5765 \\
HDPS, $x$ optimization (s)    & 0.1147 & 0.2426 & 1.1148 \\
\rowcolor{Gray}
HDPS, $z$ + $x$ optimization (s) & 0.7900 & 2.2279 & 11.6242 \\  \bottomrule
\end{tabular}
\end{table}

\begin{figure}[t]
    \centering
    \includegraphics[width=\linewidth]{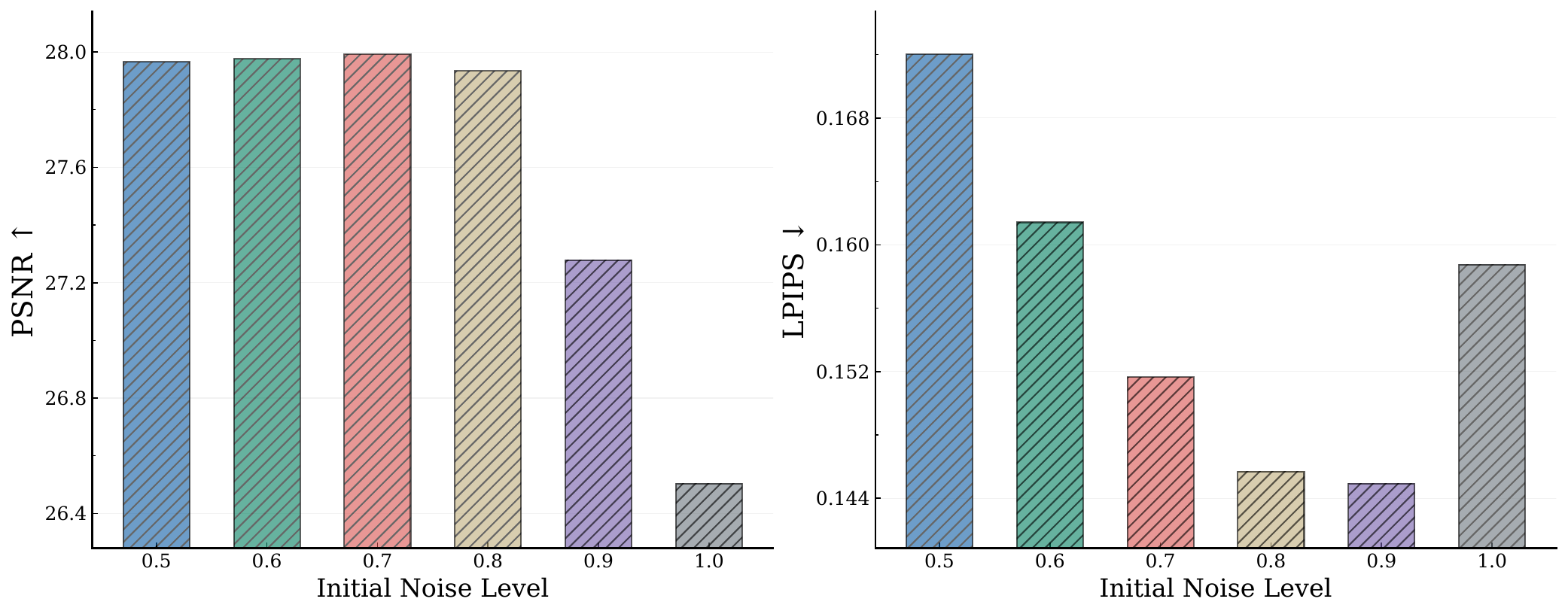}
    \caption{Sensitivity to initialization time $t_0$. $t_0=0.8$ achieves the best distortion--perception trade-off.}
    \label{fig:ablation-init}
\end{figure}

\subsection{Convergence Analysis}
\label{sec:exp_convergence}

To examine the optimization landscape, we visualize the evolution of the measurement consistency loss throughout the sampling trajectory in Figure~\ref{fig:ablation_loss}. We track both the latent-decoded error $\|\bm{y} - \mathcal{A}(D(\bm{z}_{0|t}))\|_2$ and the aligned error $\|\bm{y} - \mathcal{A}(\hat{\bm{x}}_{0|t})\|_2$.

\paragraph{Latent vs. Hybrid.}
The \textit{Latent-Only} baseline stagnates early at a high residual error, confirming that gradient updates vanish for high-frequency components due to manifold blindness.
Interestingly, the \textit{Pixel + Encoder} strategy exhibits non-monotonic behavior: in the late stages of generation, the consistency loss often increases.
This occurs because the encoder $E$, trained on clean natural images, aggressively projects the refined image $\hat{\bm{x}}$ back to the prior manifold, often undoing the subtle high-frequency corrections required to satisfy the noisy measurements.
In contrast, our proposed \textit{Pixel + Alignment} strategy continues to minimize the error monotonically, achieving a final residual orders of magnitude lower than the baselines.

\paragraph{Impact of Initialization.}
Furthermore, comparing our method with and without warm-start reveals that proper initialization significantly accelerates convergence.
The warm-start trajectory stabilizes the early Langevin dynamics, avoiding the initial chaotic search phase. This prevents early stagnation and leads to a deeper final minimum.
This validates that while the hybrid alignment mechanism is the primary driver of performance, the warm-start acts as a crucial catalyst for optimization efficiency.

\subsection{Initialization Parameter Sweep}
\label{app:init_sweep}

To systematically determine the optimal operating point for our warm-start initialization, we performed a parameter sweep varying $t_0 \in \{1.0, 0.9, \dots, 0.5\}$. The quantitative results of this sweep are presented in Figure~\ref{fig:ablation-init}.
The trends demonstrate that $t_0=0.8$ achieves the optimal trade-off between distortion (measured by PSNR) and perceptual quality (measured by LPIPS).
Earlier start times ($t_0 > 0.8$) operate in a regime where the latent code lacks structural coherence, cause the pixel-space updates to introduce unnecessary stochastic variance and artifacts. Conversely, later start times ($t_0 < 0.6$) overly constrain the generative diversity by enforcing measurements too late in the reverse process, leading to sub-optimal perceptual quality. Consequently, we adopt $t_0=0.8$ as the default and most balanced setting for HDPS.

\begin{figure}[t]
    \centering
    \includegraphics[width=\linewidth]{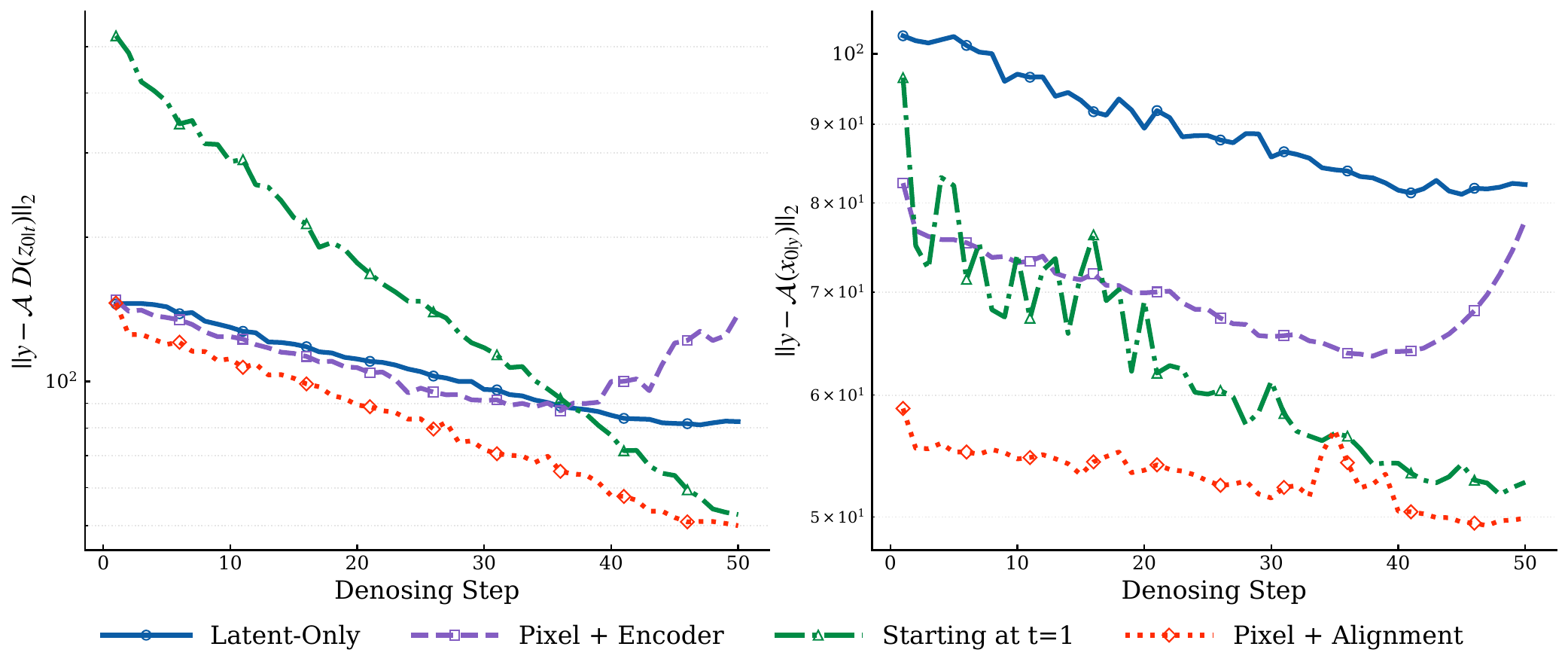}
    \caption{Convergence of measurement consistency loss. Latent-Only stagnates; Pixel + Encoder is non-monotonic; our method converges monotonically to the lowest residual.}
    \label{fig:ablation_loss}
\end{figure}

\section{Further Analysis}
\label{app:further_analysis}

\subsection{Comparison with Pixel-Space Samplers}

We complement the high-resolution latent-flow comparisons with DPS~\cite{chung2023diffusion} and DAPS~\cite{zhang2025improving}, two strong posterior samplers equipped with pixel-space diffusion priors on a subset of FFHQ. Because these pixel-space priors and our latent-flow backbone belong to different model families, Table~\ref{tab:pixel_baselines} is intended as a controlled low-resolution comparison rather than a claim of identical prior capacity. At $256\times256$, HDPS obtains the best PSNR and SSIM for random inpainting, while DAPS obtains the best LPIPS. For motion deblurring, DAPS leads in PSNR and LPIPS, whereas HDPS achieves the best SSIM. These results show that HDPS remains competitive with dedicated pixel-space samplers while retaining the latent generative prior used in our $768\times768$ experiments. Together with the LatentDAPS comparison in Table~\ref{tab:quantitative}, they also indicate that the gain does not arise from the Langevin update alone, but from coupling pixel correction to latent flow evolution through alignment.

\begin{table}[t]
\centering
\caption{Comparison with pixel-space posterior samplers on FFHQ at $256\times256$ with measurement noise $\sigma_n=0.05$. Best results are shown in bold.}
\label{tab:pixel_baselines}
\resizebox{\linewidth}{!}{
\begin{tabular}{lcccccc}
\toprule
& \multicolumn{3}{c}{Random Inpainting} & \multicolumn{3}{c}{Motion Deblurring} \\
\cmidrule(r){2-4} \cmidrule(l){5-7}
Method & PSNR$\uparrow$ & SSIM$\uparrow$ & LPIPS$\downarrow$ & PSNR$\uparrow$ & SSIM$\uparrow$ & LPIPS$\downarrow$ \\
\midrule
DPS  & 28.92 & 0.840 & 0.162 & 27.23 & 0.780 & 0.181 \\
DAPS & 30.13 & 0.799 & \textbf{0.108} & \textbf{31.15} & 0.838 & \textbf{0.103} \\
HDPS & \textbf{30.65} & \textbf{0.869} & 0.131 & 30.37 & \textbf{0.856} & 0.140 \\
\bottomrule
\end{tabular}}
\end{table}

\subsection{Explanation about Anchoring and Early Stopping}

The finite alignment can be understood as an anchored compromise between the clean flow prediction $\bm{z}_{0|t}$ and the measurement-consistent, but potentially artifact-contaminated, target $\hat{\bm{x}}_{0|t}$. Initializing at $\bm{z}_{0|t}$ and stopping after a small number of iterations implicitly regularize the solution toward the flow trajectory: the early updates absorb dominant structural corrections, whereas excessive optimization can move the latent farther from its anchor and fit Langevin artifacts. Thus, the decoder's eventual ability to fit an artifact concerns representational capacity and does not contradict the first-order accessibility bottleneck motivating the pixel correction. 



\subsection{Posterior Diversity}

HDPS retains stochasticity at two points: the Langevin noise $\bm{\xi}_j$ in Eq.~\eqref{eq:langevin} explores measurement-consistent pixel corrections, and the independent perturbation $\bm{\epsilon}'$ in Eq.~\eqref{eq:state_update} injects stochasticity when returning to the latent trajectory. The deterministic alignment can contract some of this variation, while anchoring and early stopping prevent unrestricted drift. Consequently, $N_z$ and the state-update noise schedule jointly mediate the fidelity--diversity trade-off. Our present evaluation focuses on reconstruction fidelity and perceptual quality; a full characterization of conditional uncertainty, including multi-sample pairwise diversity and calibrated uncertainty maps, remains an important direction for future work.

\subsection{Extension to Nonlinear and Real-World Degradations}

The hybrid construction is not restricted in form to a linear operator. For a known differentiable nonlinear forward model $\mathcal{A}$ with Gaussian measurement noise, the likelihood term in the pixel Langevin step becomes
\begin{equation}
  \nabla_{\bm{x}}\log p(\bm{y}\mid\bm{x})
  = \frac{J_{\mathcal{A}}(\bm{x})^\top\big(\bm{y}-\mathcal{A}(\bm{x})\big)}{\sigma_y^2},
\end{equation}
which can replace the linear adjoint term in Eq.~\eqref{eq:langevin} without changing the latent alignment or flow update. If the physical operator is non-differentiable but paired data or a simulator is available, a differentiable neural surrogate $\widetilde{\mathcal{A}}$ can provide an approximate likelihood gradient, with reconstruction quality depending on the surrogate's accuracy. Blind and real-world restoration are more demanding: the operator, its parameters, and the noise distribution may all be unknown, so differentiability alone is insufficient. Extending HDPS to these settings would require joint or alternating estimation of the degradation model and the image, as well as robustness to operator and noise mismatch; we regard this as a promising extension rather than an empirically established capability of the current model.

\section{Additional Qualitative Results}
\label{app:qualitative}


In this section, we provide supplementary qualitative results to further substantiate the effectiveness of our proposed method. 
\Cref{fig:more_mc_1,fig:more_mc_2,fig:more_mc_3,fig:more_mc_4,fig:more_mc_5} 
display additional visual comparisons on random inpainting, Gaussian deblurring, motion deblurring, 
$\times 12$ bicubic super-resolution, and $\times 12$ average pooling super-resolution, respectively. 
These extensive results further demonstrate the superior generative capability of our approach 
in producing high-fidelity textures while strictly adhering to data consistency across various 
degradation types.

\begin{figure*}[ht]
    \centering
    \begin{tikzpicture}
        \node[anchor=south west, inner sep=0] (image) at (0,0) {
            \includegraphics[width=0.9\linewidth]{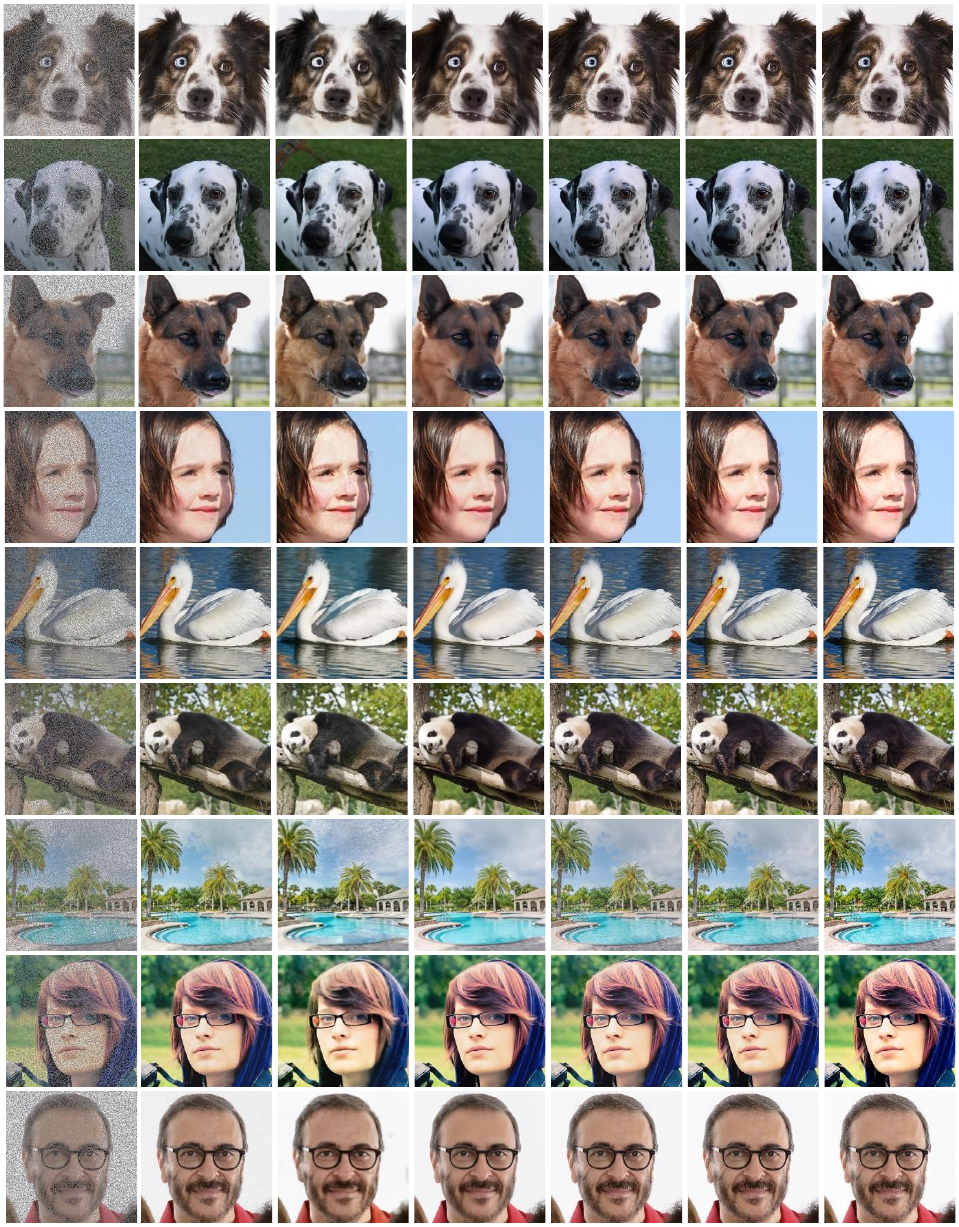}
        };
        \begin{scope}[x={(image.south east)},y={(image.north west)}]
            \node at (0.075, 1.01) {\small Input};
            \node at (0.21, 1.01)  {\small Resample};
            \node at (0.355, 1.01) {\small FlowChef};
            \node at (0.495, 1.01) {\small FlowDPS};
            \node at (0.64, 1.01)  {\small FLAIR};
            \node at (0.78, 1.01)  {\small \textbf{Ours}};
            \node at (0.92, 1.01)  {\small Reference};
        \end{scope}
    \end{tikzpicture}
    \caption{Additional Qualitative results to illustrate the effectiveness of our proposed method on random inpainting task with $\sigma_y=0.03$.} 
    \label{fig:more_mc_1} 
\end{figure*}
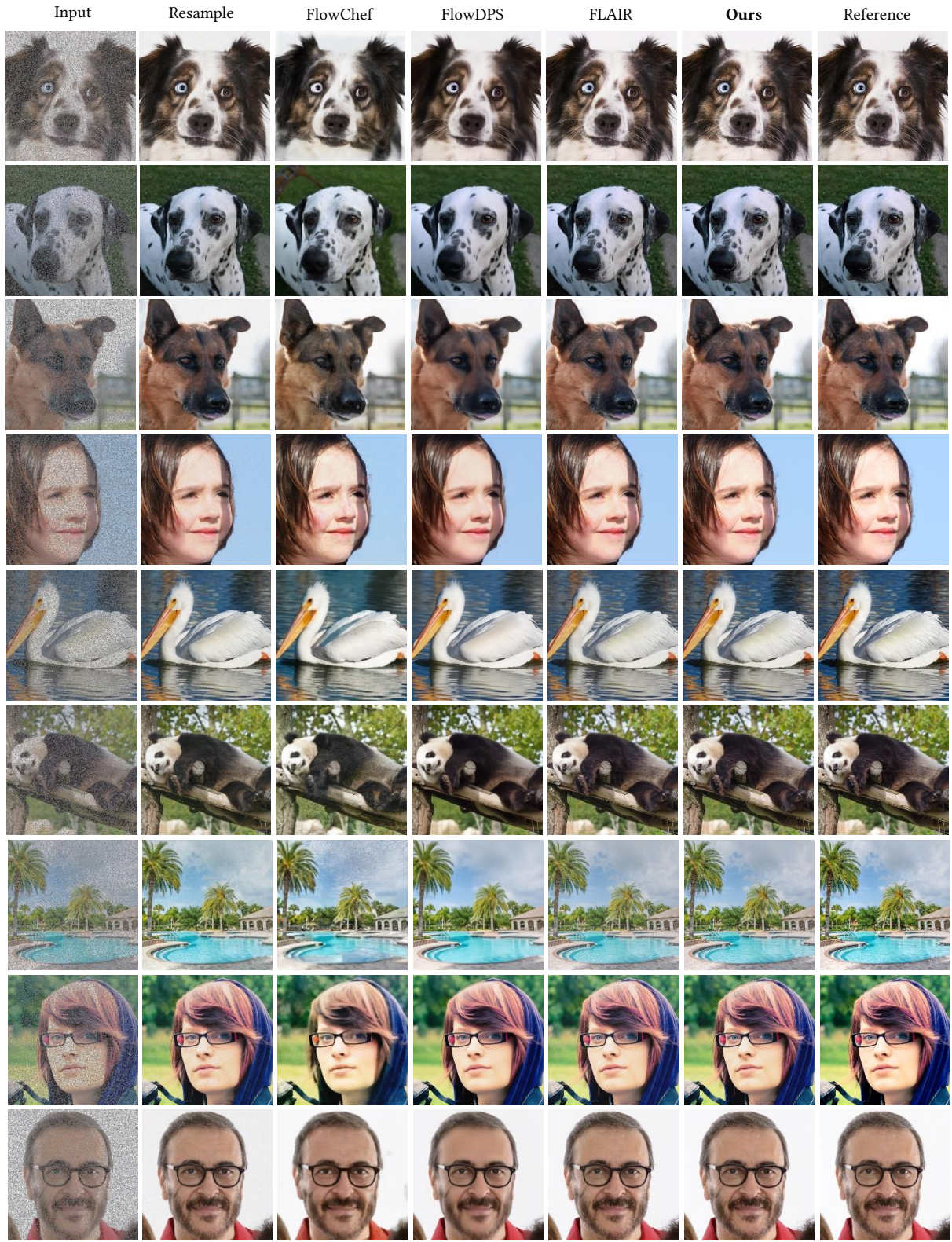

\begin{figure*}[ht]
    \centering
    \begin{tikzpicture}
        \node[anchor=south west, inner sep=0] (image) at (0,0) {
            \includegraphics[width=0.9\linewidth]{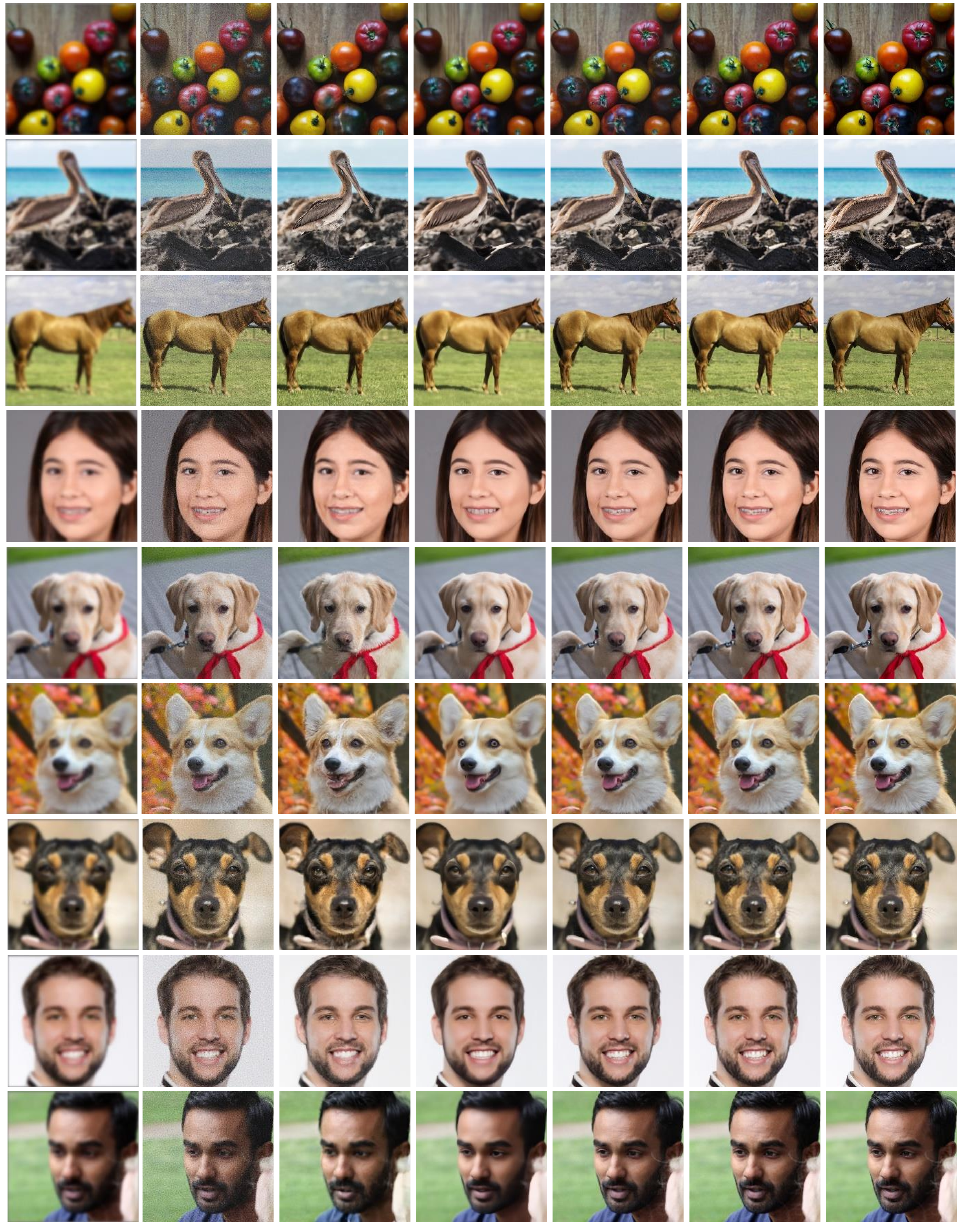}
        };
        \begin{scope}[x={(image.south east)},y={(image.north west)}]
            \node at (0.075, 1.01) {\small Input};
            \node at (0.21, 1.01)  {\small Resample};
            \node at (0.355, 1.01) {\small FlowChef};
            \node at (0.495, 1.01) {\small FlowDPS};
            \node at (0.64, 1.01)  {\small FLAIR};
            \node at (0.78, 1.01)  {\small \textbf{Ours}};
            \node at (0.92, 1.01)  {\small Reference};
        \end{scope}
    \end{tikzpicture}
    \caption{Additional Qualitative results to illustrate the effectiveness of our proposed method on Gaussian deblurring task with $\sigma_y=0.03$.} 
    \label{fig:more_mc_2} 
\end{figure*}

\begin{figure*}[ht]
    \centering
    \begin{tikzpicture}
        \node[anchor=south west, inner sep=0] (image) at (0,0) {
            \includegraphics[width=0.9\linewidth]{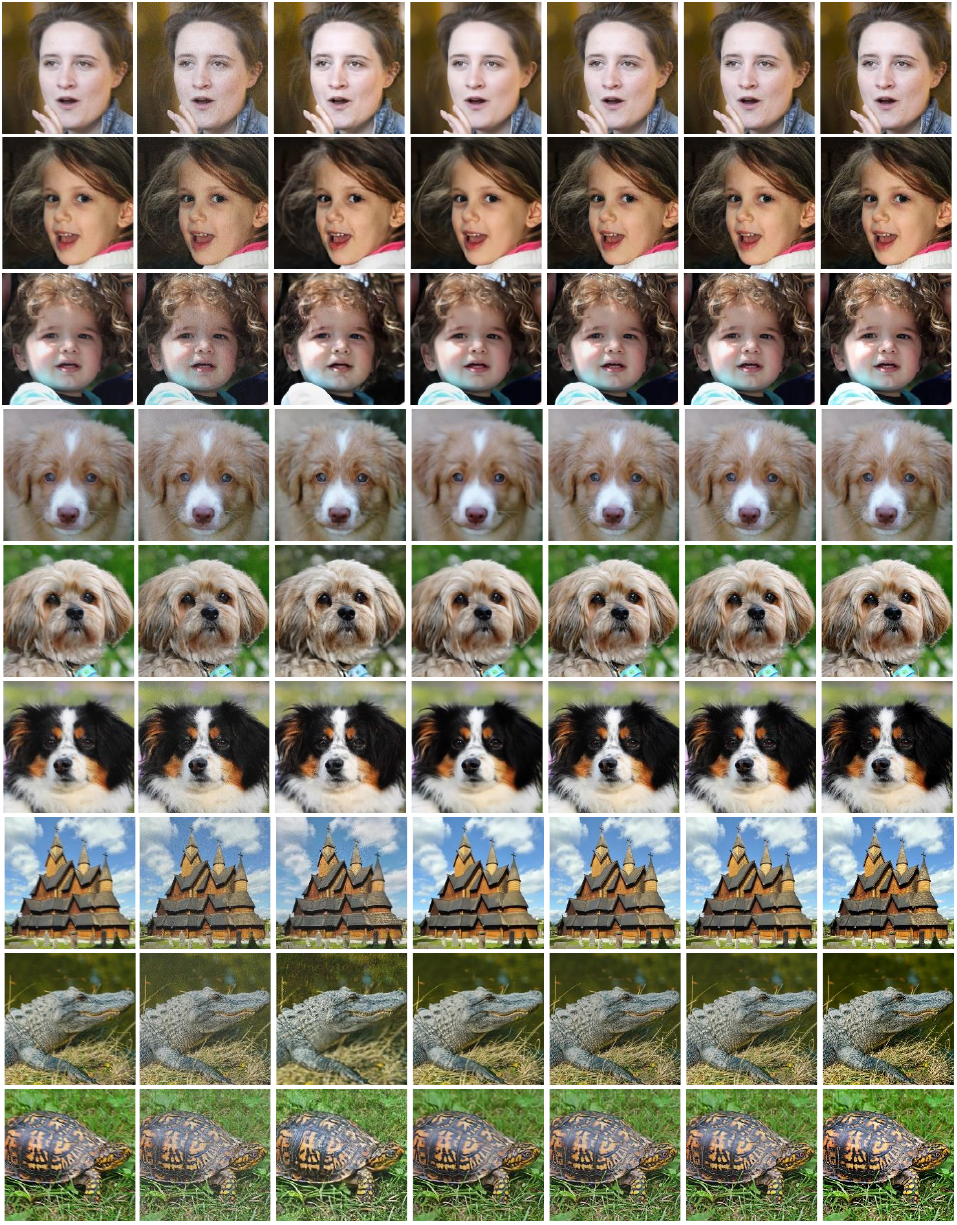}
        };
        \begin{scope}[x={(image.south east)},y={(image.north west)}]
            \node at (0.075, 1.01) {\small Input};
            \node at (0.21, 1.01)  {\small Resample};
            \node at (0.355, 1.01) {\small FlowChef};
            \node at (0.495, 1.01) {\small FlowDPS};
            \node at (0.64, 1.01)  {\small FLAIR};
            \node at (0.78, 1.01)  {\small \textbf{Ours}};
            \node at (0.92, 1.01)  {\small Reference};
        \end{scope}
    \end{tikzpicture}
    \caption{Additional Qualitative results to illustrate the effectiveness of our proposed method on motion deblurring task with $\sigma_y=0.03$.} 
    \label{fig:more_mc_3} 
\end{figure*}

\begin{figure*}[ht]
    \centering
    \begin{tikzpicture}
        \node[anchor=south west, inner sep=0] (image) at (0,0) {
            \includegraphics[width=0.9\linewidth]{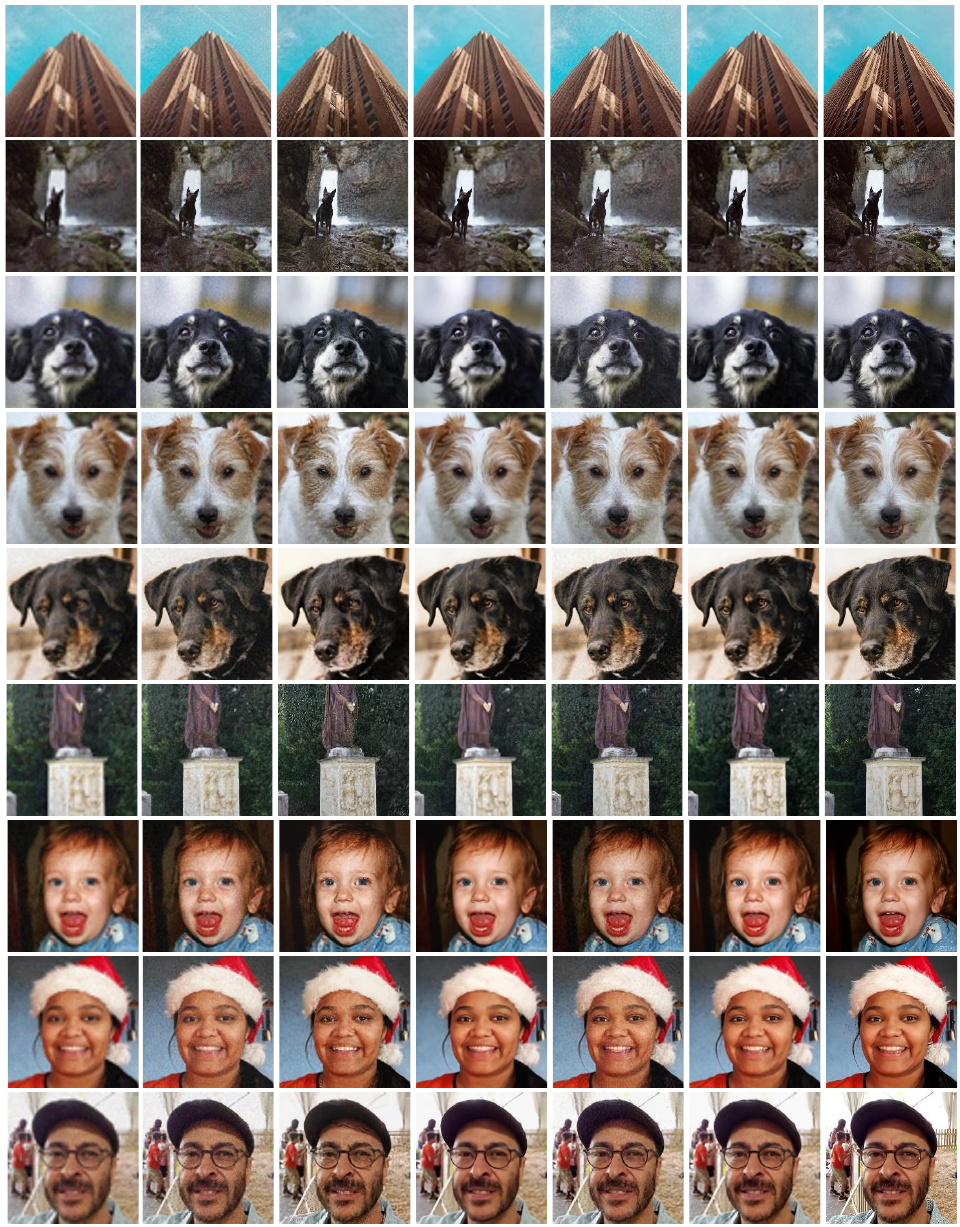}
        };
        \begin{scope}[x={(image.south east)},y={(image.north west)}]
            \node at (0.075, 1.01) {\small Input};
            \node at (0.21, 1.01)  {\small Resample};
            \node at (0.355, 1.01) {\small FlowChef};
            \node at (0.495, 1.01) {\small FlowDPS};
            \node at (0.64, 1.01)  {\small FLAIR};
            \node at (0.78, 1.01)  {\small \textbf{Ours}};
            \node at (0.92, 1.01)  {\small Reference};
        \end{scope}
    \end{tikzpicture}
    \caption{Additional Qualitative results to illustrate the effectiveness of our proposed method on $\times 12$ super-resolution task from bicubic downsampling with $\sigma_y=0.03$.} 
    \label{fig:more_mc_4} 
\end{figure*}
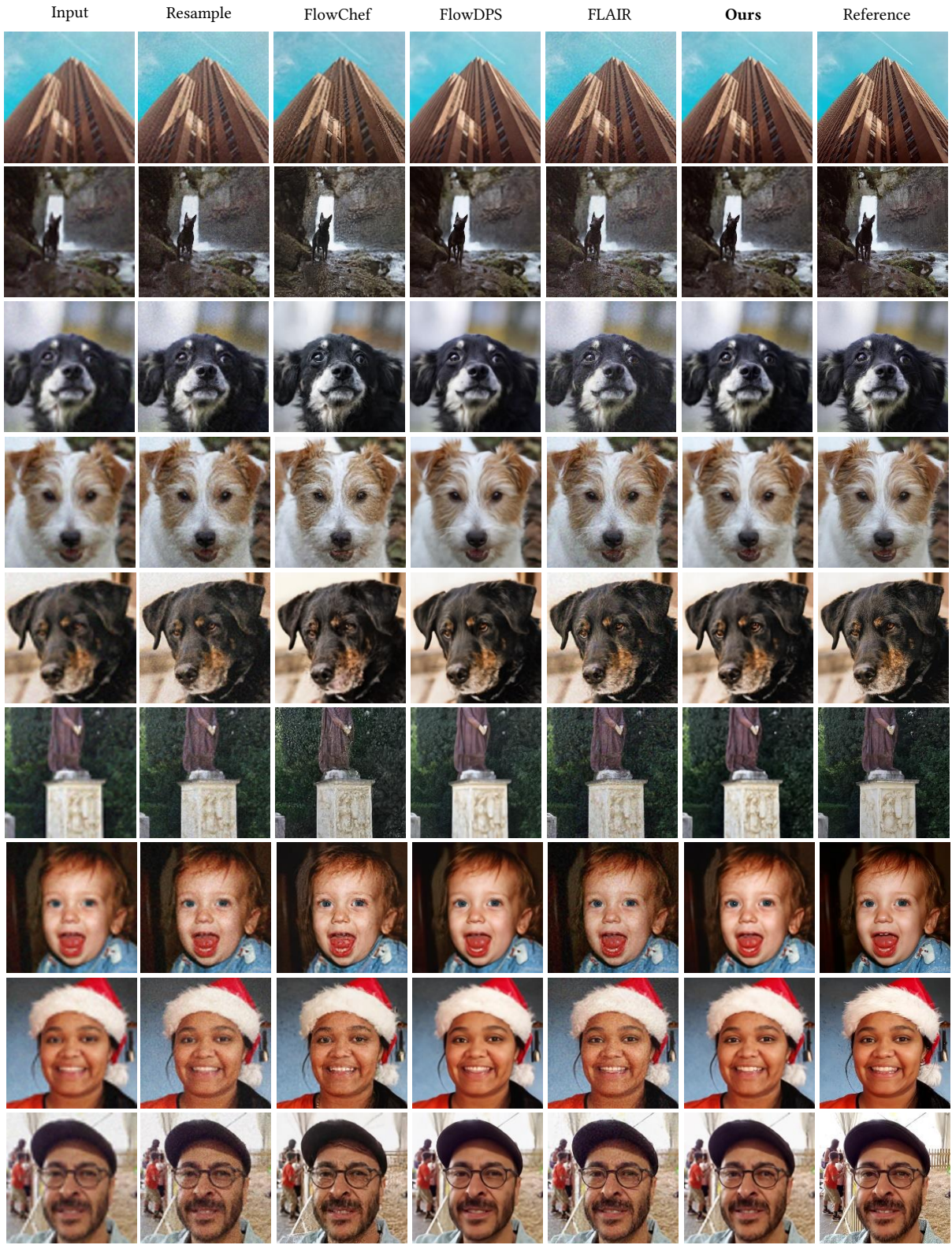

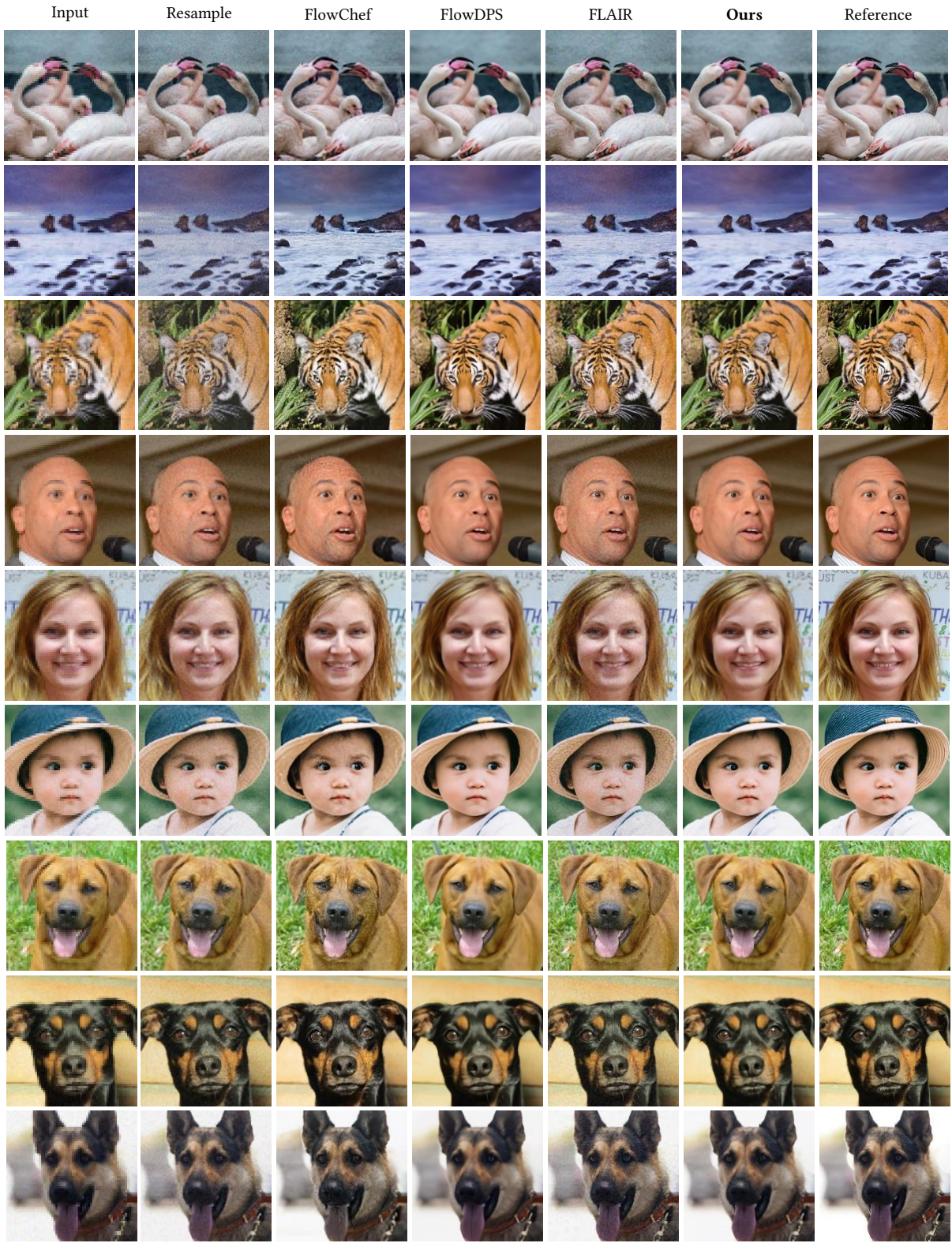
\begin{figure*}[ht]
    \centering
    \begin{tikzpicture}
        \node[anchor=south west, inner sep=0] (image) at (0,0) {
            \includegraphics[width=0.9\linewidth]{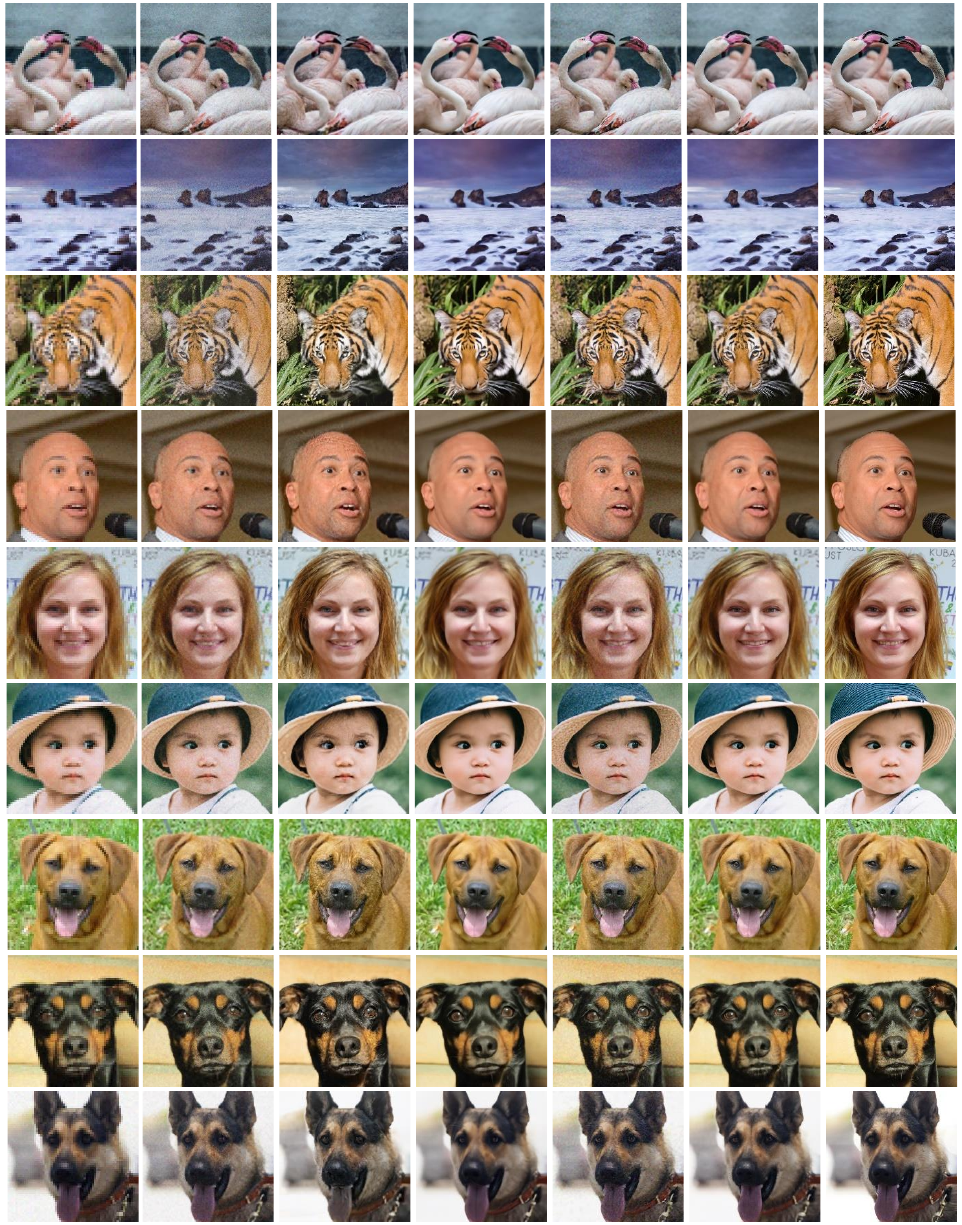}
        };
        \begin{scope}[x={(image.south east)}, y={(image.north west)}]
            \node at (0.075, 1.01) {\small Input};
            \node at (0.21, 1.01)  {\small Resample};
            \node at (0.355, 1.01) {\small FlowChef};
            \node at (0.495, 1.01) {\small FlowDPS};
            \node at (0.64, 1.01)  {\small FLAIR};
            \node at (0.78, 1.01)  {\small \textbf{Ours}};
            \node at (0.92, 1.01)  {\small Reference};
        \end{scope}
    \end{tikzpicture}
    \caption{Additional Qualitative results to illustrate the effectiveness of our proposed method on $\times 12$ super-resolution task from average pooling with $\sigma_y=0.03$.} 
    \label{fig:more_mc_5} 
\end{figure*}

\end{document}